\documentclass[12pt]{article}
\usepackage[letterpaper,margin=1in]{geometry}

\usepackage{mathtools}
\usepackage{booktabs}
\usepackage[english]{babel} 
\usepackage[protrusion=true,expansion=true]{microtype} 
\usepackage{amsmath,amsfonts,amsthm}
\allowdisplaybreaks[4]
\usepackage{ amssymb }
\usepackage{enumitem}
\usepackage{graphicx}
\usepackage{array}
\usepackage[toc,page]{appendix}
\usepackage{xcolor}
\usepackage{tikz}
\usetikzlibrary{arrows.meta,positioning}

\usepackage{booktabs} 
\usepackage{natbib}
\usepackage{setspace}
\usepackage{microtype} 
\usepackage{tabularx}
\usepackage{subcaption}
\usepackage[font = small,labelfont=bf,textfont=it]{caption} 
\usepackage{footnote}
\usepackage{algorithm}
\usepackage[noend]{algpseudocode}
\usepackage{multirow}
\usepackage[version=4]{mhchem}
\usepackage[hidelinks]{hyperref}

\makeatletter
\newcommand{\distas}[1]{\mathbin{\overset{#1}{\kern\z@\sim}}}%
\newsavebox{\mybox}\newsavebox{\mysim}
\newtheorem{theorem}{Theorem}[section]
\newtheorem{corollary}[theorem]{Corollary}

\newtheorem{lemma}[theorem]{Lemma}
\theoremstyle{definition}

\newtheorem{assumption}{Assumption}[section]

\newtheorem{remark}[theorem]{Remark}
\newcommand{\distras}[1]{%
  \savebox{\mybox}{\hbox{\kern3pt$\scriptstyle#1$\kern3pt}}%
  \savebox{\mysim}{\hbox{$\sim$}}%
  \mathbin{\overset{#1}{\kern\z@\resizebox{\wd\mybox}{\ht\mysim}{$\sim$}}}%
}
\newcolumntype{C}[1]{>{\centering\let\newline\\\arraybackslash\hspace{0pt}}m{#1}}

\newcommand{\blind}{1}

\newcommand{\bW}{\mathbf{W}}
\newcommand{\bX}{\mathbf{X}}
\newcommand{\bY}{\mathbf{Y}}
\newcommand{\bU}{\mathbf{U}}
\newcommand{\bV}{\mathbf{V}}
\newcommand{\bk}{\mathbf{k}}
\newcommand{\btheta}{\boldsymbol{\theta}}
\newcommand{\cL}{\mathcal{L}}

\newcommand{\bh}{\mathbf{h}}
\newcommand{\bm}{\mathbf{m}}
\newcommand{\bv}{\mathbf{v}}
\newcommand{\bx}{\mathbf{x}}
\newcommand{\bw}{\mathbf{w}}
\newcommand{\bd}{\mathbf{d}}

\begin{document}

\def\spacingset#1{\renewcommand{\baselinestretch}%
{#1}\small\normalsize} \spacingset{1.3}


\if1\blind
{
 \centering{\bf\Large 
 Efficient Deep Gaussian Process Surrogates via Mini-Batch Latent MAP Estimation}\\
  \vspace{0.3in}
  \centering{Chun-Yi Chang and Chih-Li Sung\vspace{0.1in}\\
        Michigan State University\\
        }
    \date{\vspace{-7ex}}
} \fi

\if0\blind
{
  \bigskip
  \bigskip
  \bigskip
    \begin{center}
    {\LARGE\bf Efficient Deep Gaussian Process Surrogates via Mini-Batch Latent MAP Estimation}
\end{center}
  \medskip
} \fi

\bigskip
\begin{abstract}
    Deep Gaussian processes (DGPs) are flexible surrogate models for computer experiments with nonstationary response surfaces, but inference over their latent representations remains computationally challenging. Existing methods typically rely on variational inference or Markov chain Monte Carlo (MCMC), which impose variational-family assumptions or require costly posterior sampling. We propose the latent maximum-a-posteriori deep Gaussian process (LaMP-DGP), which uses gradient-based optimization to obtain MAP estimates of the latent representations and enables efficient probabilistic prediction through linked-GP calculations.  To scale LaMP-DGP to large datasets, we develop a mini-batch optimization algorithm and, under a polynomial mini-batch bias-decay condition, establish a stationarity bound for latent-state optimization that motivates a growing batch-size schedule. Across numerical experiments and real-data analyses, LaMP-DGP provides consistently competitive predictive performance relative to doubly stochastic variational inference, stochastic gradient Hamiltonian Monte Carlo, and fully Bayesian Vecchia DGP inference. Computationally, LaMP-DGP is substantially faster than the Vecchia MCMC approach while remaining broadly competitive with the stochastic-gradient alternatives. These results position LaMP-DGP as a competitive and computationally efficient alternative for DGP inference. An \textsf{R} package, \texttt{LaMPDGP}, is available for implementing the proposed method.
\end{abstract}

\noindent%
{\it Keywords}: computer experiments; nonstationary Gaussian process; uncertainty quantification; latent maximum a posteriori; mini-batch optimization
\vfill

\newpage
\spacingset{1.45} 

\section{Introduction}
Computer simulation plays a crucial role across a wide range of fields, including aerospace engineering \citep{mak2018efficient}, physics \citep{MORITA2022110788}, and cell biology \citep{sung2020generalized}, where obtaining precise solutions is essential, yet physical experiments are often prohibitively expensive. To address this challenge, surrogate models  have emerged as computationally efficient approaches for approximating complex systems \citep{gramacy2020surrogates}. Gaussian Processes (GPs) \citep{rasmussen2006gaussian} are one of the methods that have gained significant attention due to their flexibility and ability to capture nonlinear and complex functional relationship. Moreover, the inherent uncertainty quantification makes GPs particularly valuable for sensitivity analysis, model calibration, and sequential design, thereby supporting informed decision-making in real-world applications.

Despite their advantages, standard GPs often struggle with non-stationary problems because they typically rely on a single stationary kernel, which assumes a uniform level of smoothness across the input space. This restrictive assumption can oversmooth rapidly varying regions and consequently lead to poor predictive performance. Deep Gaussian processes (DGPs) address this limitation by composing multiple GP mappings within a hierarchical architecture, enabling flexible modeling of complex and nonstationary input–output relationships. DGPs were introduced by \citet{damianou2013deep}, who developed a variational inference (VI) framework \citep{blei2017variational} that approximately marginalizes the intermediate latent variables through variational distributions, thereby enabling multiple GP layers to be composed. 

Since then, numerous inference methods have been developed to improve the predictive performance and computational scalability of DGPs. One notable extension of VI-based approaches is the doubly stochastic variational inference (DSVI) method proposed by \cite{salimbeni2017doubly}. By combining inducing-point approximations with stochastic optimization over data mini-batches and Monte Carlo propagation through the GP layers, DGP-DSVI substantially improves the scalability of variational inference for DGPs. However, its approximation quality depends on the variational family, the inducing-point representation, and the Monte Carlo approximation, which can affect predictive accuracy and uncertainty quantification.

As an alternative to variational inference, \citet{DGPHMC2018} proposed stochastic gradient Hamiltonian Monte Carlo (SGHMC) for DGPs, using SGHMC to sample the posterior distribution of the inducing outputs and thereby allowing more flexible non-Gaussian posterior structure than Gaussian variational approximations. A different sampling-based approach was developed by \citet{sauer2023active}, who proposed a fully Bayesian MCMC framework that samples both the latent representations and covariance hyperparameters, providing more comprehensive posterior uncertainty quantification at a substantially higher computational cost. To improve scalability, \citet{sauer2023vecchia} subsequently incorporated the Vecchia approximation \citep{katzfuss2021general} into this MCMC framework while preserving the advantages of fully Bayesian inference. We refer to these approaches as DGP-SGHMC and MCMC-VEC, respectively, throughout the remainder of the paper.

Motivated by the need for computationally efficient DGP inference without specifying a variational family or running lengthy MCMC chains over the training latent states, we propose the latent maximum-a-posteriori deep Gaussian process (LaMP-DGP). LaMP-DGP uses gradient-based optimization to obtain a maximum-a-posteriori (MAP) estimate of the latent representation while jointly learning the kernel hyperparameters. Under the Gaussian likelihood and squared-exponential kernels considered here, we derive analytical gradients for the MAP objective, thereby avoiding costly finite-difference approximations.  For prediction, conditional on the fitted latent states and covariance hyperparameters, we follow the linked-GP framework of \citet{Ming03042023} to analytically integrate over uncertainty in the test-point latent representation, enabling efficient probabilistic prediction without prediction-time Monte Carlo sampling.

Rather than approximating or sampling the entire posterior distribution, MAP identifies a latent configuration with high posterior density given the observed inputs and responses. The likelihood favors configurations that explain the observed responses, whereas the GP prior regularizes the latent representation according to the covariance structure induced by the inputs. LaMP-DGP therefore recasts inference over the training latent states as a regularized optimization problem. This MAP perspective has an established precedent in Gaussian process latent-variable models (GP-LVMs), where latent coordinates are optimized as point estimates under prior regularization \citep{lawrence2003gaussian,JMLR:v6:lawrence05a}, with extensions to structured priors and multiple latent layers \citep{NIPS2005_ccd45007,lawrence2007hierarchical}. LaMP-DGP builds on this perspective in the supervised regression setting, optimizing GP-regularized latent states directly from the observed input--response pairs.

MAP inference eliminates posterior sampling, but for larger datasets each objective evaluation still requires factorizing covariance matrices whose dimensions grow with the training sample size. To address this computational bottleneck, we develop a mini-batch optimization algorithm for LaMP-DGP that updates the latent states and kernel hyperparameters using subsets of the training data. Unlike standard empirical-risk objectives that decompose additively across observations, the LaMP-DGP negative log-posterior depends jointly on the observations through covariance matrices, so its mini-batch gradients are generally biased \citep{chen2022gaussian}. Under suitable regularity conditions and a polynomial bias-decay condition, we derive a nonconvex stationarity bound that separates the usual optimization error from the additional error induced by mini-batch bias. This analysis motivates a progressively increasing mini-batch schedule that reduces this bias while controlling computational cost and, under the stated conditions, preserves an overall \(\mathcal{O}(T^{-1/2})\) stationarity rate, where \(T\) denotes the number of optimization epochs.

Within this framework, our contributions are threefold: (i) we introduce LaMP-DGP, a new optimization-based inference framework for DGPs that directly estimates the posterior modes of the training latent representations, avoiding both variational-family approximations and MCMC sampling, while enabling efficient probabilistic prediction through linked-GP calculations; (ii) we develop a mini-batch optimization algorithm for larger datasets and, under a polynomial bias-decay condition, establish a nonconvex stationarity bound that motivates a progressively increasing mini-batch schedule; and (iii) through extensive numerical experiments and real-world computer experiments, we demonstrate that LaMP-DGP provides consistently competitive predictive performance across a broad range of problems while substantially reducing computational cost relative to fully Bayesian Vecchia DGP inference. We provide an open-source implementation in the \texttt{LaMPDGP} package for \textsf{R}, available in the supplementary material.

The remainder of this paper is organized as follows. Section~\ref{sec:DGP} reviews DGPs. Section~\ref{sec:LaMP-DGP} introduces the proposed LaMP-DGP framework. Section~\ref{sec:LaMP-DGP SGD} presents the mini-batch algorithm together with its convergence analysis. Section~\ref{sec:numerical} evaluates LaMP-DGP using a collection of synthetic benchmark functions, and Section~\ref{sec:real data} further assesses its performance on real-world computer experiments. Finally, Section~\ref{sec:conclusion} concludes the paper with a discussion of the main findings and future directions.

\section{Deep Gaussian Processes (DGPs)} \label{sec:DGP}

Let \(f:\mathbb{R}^d\to\mathbb{R}\) denote a computationally expensive
black-box function, and suppose the observed output at input \(\bx\) is
\[
y=f(\bx)+\epsilon,
\qquad
\epsilon\sim\mathcal{N}(0,\sigma_\epsilon^2).
\]
Given a design matrix
\(
\bX=(\bx_1,\ldots,\bx_n)^\top\in\mathbb{R}^{n\times d},
\)
let
\(
Y=
\big(
f(\bx_1)+\epsilon_1,\ldots,
f(\bx_n)+\epsilon_n
\big)^\top
\in\mathbb{R}^n
\)
denote the corresponding observed outputs, where
\(
\epsilon_1,\ldots,\epsilon_n
\overset{\mathrm{iid}}{\sim}
\mathcal{N}(0,\sigma_\epsilon^2).
\)
The deterministic case is recovered by setting
\(\sigma_\epsilon^2=0\). Gaussian process (GP) models are widely used as surrogates for modeling $f$ because they provide both flexible prediction and uncertainty quantification \citep{gramacy2020surrogates}. Under a stationary GP, the covariance between two outputs depends on the inputs only through their relative separation; that is, \(\operatorname{Cov}\{f(\bx),f(\bx')\}\) depends on $\bx-\bx'$.

While effective for many applications, this assumption can be restrictive when the response surface exhibits spatially varying smoothness, local features, or other forms of nonstationarity \citep{booth2023nonstationary}. In such cases, a single stationary covariance structure in the original input space may not adequately represent the complexity of \(f\). Deep Gaussian processes (DGPs) provide an alternative to conventional stationary GP surrogates by introducing latent GP layers, allowing more flexible representations of nonstationary response surfaces. Here, we consider a two-layer DGP, with its hierarchical structure and associated parameters summarized in Figure~\ref{fig:LaMP-DGP}.

\begin{figure}[t]
    \centering

    \resizebox{0.65\textwidth}{!}{
    \begin{tikzpicture}[
    node distance=3cm and 3cm,
    roundnode/.style={circle, draw=blue!40, fill=blue!5, very thick, minimum size=10mm},
    annotation/.style={draw=none, font=\small},
    arrow/.style={-{Latex[length=2mm]}, thick},
    line/.style={-, thick},
    roundedbox/.style={
        draw,
        rectangle,
        rounded corners=6pt,
        minimum height=8mm,
        minimum width=12mm,
        align=center}
    ]

    \node[roundnode](X){$\bX$};
    \node[roundedbox](paraW1)[right=1cm of X, yshift=2cm]{$f_{11}$};
    \node[roundedbox](paraW2)[right=1cm of X, yshift=0.4cm]{$f_{12}$};
    \node[roundedbox](paraWD)[right=1cm of X, yshift=-2.5cm]{$f_{1D}$};

    \node[roundnode](W1)[right=1.5cm of paraW1]{$W_{\cdot1}$};
    \node[roundnode](W2)[right=1.5cm of paraW2]{$W_{\cdot2}$};
    \node[roundnode](WD)[right=1.5cm of paraWD]{$W_{\cdot D}$};

    \node[roundedbox](paraY)[right=1.2cm of W2, yshift=-0.4cm]{$f_2$};
    \node[roundnode](Y)[right=1.8cm of paraY]{$Y$};

    \node[draw=none, below=0.2cm of W2] (dots1) {$\vdots$};
    \node[draw=none, below=0.35cm of paraW2] (dots4) {$\vdots$};

    \draw[line] (X) -- (paraW1);
    \draw[line] (X) -- (paraW2);
    \draw[line] (X) -- (paraWD);

    \draw[arrow] (paraW1) -- node[midway, above] {$\theta_{1,1}$} (W1);
    \draw[arrow] (paraW2) -- node[midway, above] {$\theta_{1,2}$} (W2);
    \draw[arrow] (paraWD) -- node[midway, above] {$\theta_{1,D}$} (WD);

    \draw[line] (W1) -- (paraY);
    \draw[line] (W2) -- (paraY);
    \draw[line] (WD) -- (paraY);

    \draw[arrow] (paraY) -- node[midway, above] {$\theta_{2}, \tau^2, g$} (Y);

    \end{tikzpicture}
    }

    \caption{Hierarchical structure of a two-layer DGP. The first-layer GP mappings \(f_{1j}\), \(j=1,\ldots,D\), with covariance parameters \(\theta_{1,j}\), map the inputs \(\bX\) to the latent representations \(W_{\cdot j}\). The second-layer GP \(f_2\), parameterized by \(\theta_2\), \(\tau^2\), and nugget \(g\), maps the latent representation \(W=(W_{\cdot1},\ldots,W_{\cdot D})\) to the observed response \(Y\).}
    \label{fig:LaMP-DGP}
    \end{figure}
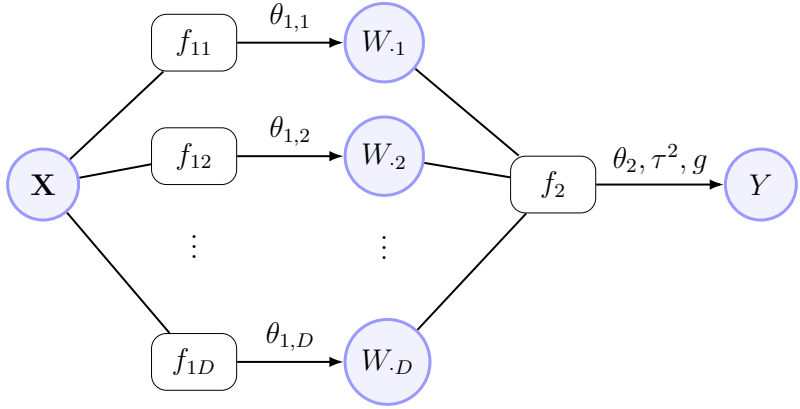

Specifically, a two-layer DGP can be represented as
\[
f(\bx)
=
f_2\big(
f_{11}(\bx),\ldots,f_{1D}(\bx)
\big),
\]
where each \(f_{1j}:\mathbb{R}^d\to\mathbb{R}\) maps the original input
to a latent variable, and
\(f_2:\mathbb{R}^D\to\mathbb{R}\) maps the resulting
\(D\)-dimensional latent representation to the response. We denote the
evaluations of the \(j\)th latent GP at the design inputs by
\(
W_{\cdot j}
=
f_{1j}(\bX)
:=
\big(
f_{1j}(\bx_1),\ldots,f_{1j}(\bx_n)
\big)^\top,
\)
and collect them into the latent matrix
\(
\bW
=
(W_{\cdot1},\ldots,W_{\cdot D})
\in\mathbb{R}^{n\times D}.
\)
Using the same convention,
\(
f(\bX)=f_2(\bW).
\)

We assign GP priors to both layers, with
$f_{11},\ldots,f_{1D}$ modeled as mutually independent GPs and
$f_2$ assigned a separate GP prior:
\begin{align*}
f_{1j}(\bx)
&\overset{\mathrm{ind}}{\sim}
\mathcal{GP}\!\left(
0,K_{\theta_{1j}}(\bx,\bx')
\right),
\qquad j=1,\ldots,D,                                     \\
f_2(\bw)
&\sim
\mathcal{GP}\!\left(
0,\tau^2 K_{\theta_2}(\bw,\bw')
\right).                                             
\end{align*}
For simplicity, both layers are assumed to have zero mean functions. Throughout this paper, we use squared-exponential kernels,
\begin{align}\label{eq:exkernel}
K_{\theta_{1j}}(\bx,\bx')
=
\exp\!\left\{
-\frac{\|\bx-\bx'\|^2}{\theta_{1j}}
\right\},         \qquad                               
K_{\theta_2}(\bw,\bw')
=
\exp\!\left\{
-\frac{\|\bw-\bw'\|^2}{\theta_2}
\right\},                                         
\end{align}
although other covariance functions, such as Mat\'ern kernels, can be
used within the same framework. Following \cite{sauer2023active}, we fix the first-layer process
variances at one to avoid confounding between the scale of the latent
representation and the second-layer correlation parameter $\theta_2$.
The second layer retains a process variance parameter $\tau^2$.

Accordingly, at the design inputs $\bX$,
\[
W_{\cdot j}=f_{1j}(\bX)
\overset{\mathrm{ind}}{\sim}
\mathcal N_n\!\left(
\mathbf 0,K_{\theta_{1j}}(\bX)
\right),
\qquad
[K_{\theta_{1j}}(\bX)]_{ik}
=
K_{\theta_{1j}}(\bx_i,\bx_k),
\]
and, conditional on $\bW$,
\[
Y\mid\bW
\sim
\mathcal N_n
\left(
\mathbf 0,
\tau^2\{K_{\theta_2}(\bW)+g\mathbf I_n\}
\right),
\qquad
[K_{\theta_2}(\bW)]_{ik}
=
K_{\theta_2}(\bw_i,\bw_k),
\]
where $g=\sigma_\epsilon^2/\tau^2$ is the nugget-to-process variance
ratio, a convenient parameterization for covariance estimation
\citep{gramacy2020surrogates}.

By the conditional independence of the latent processes, the joint density
factorizes as
\[
p(Y,\bW\mid\bX,\boldsymbol{\theta}_1,\theta_2,\tau^2,g)
=
p(Y\mid\bW,\theta_2,\tau^2,g)
\prod_{j=1}^D
p(W_{\cdot j}\mid\bX,\theta_{1j}),
\]
where
\(\boldsymbol{\theta}_1=(\theta_{11},\ldots,\theta_{1D})\).
Up to additive constants, the corresponding log-density terms are
\begin{align*}
\log p(\bW\mid\bX,\boldsymbol{\theta}_1)
&=
-\frac12\sum_{j=1}^D
\left\{
W_{\cdot j}^\top
K_{\theta_{1j}}(\bX)^{-1}
W_{\cdot j}
+
\log\left|K_{\theta_{1j}}(\bX)\right|
\right\},
\\
\log p(Y\mid\bW,\theta_2,\tau^2,g)
&=
-\frac{n}{2}\log\tau^2
-\frac12
\log\left|K_{\theta_2}(\bW)+g\mathbf I_n\right| 
-\frac{1}{2\tau^2}
Y^\top
\left\{K_{\theta_2}(\bW)+g\mathbf I_n\right\}^{-1}
Y.
\end{align*}
Exact Bayesian inference for DGPs requires marginalizing over the latent
representation. In particular, the marginal likelihood is obtained by
integrating out $\bW$:
\begin{align}
p(Y\mid\bX)
=
\int
p(Y\mid\bW)\,
p(\bW\mid\bX)\,
d\bW,
\label{eq:marginal likelihood}
\end{align}
where dependence on the covariance parameters is suppressed for notational
simplicity. This integral is generally intractable because of the nonlinear
dependence between layers and the high dimensionality of the latent
representation.

Similarly, exact Bayesian prediction of the underlying response
$f_*:=f(\bx_*)$ at a new input $\bx_*$ requires integrating over the
latent representations at both the training and test inputs. Let $\bw_*$ denote the latent representation corresponding to
$\bx_*$. The predictive distribution can be written as
\begin{equation}\label{eq:prediction_original}
p(f_*\mid\bx_*,\bX,Y)
=
\iint
p(f_*\mid\bw_*,\bW,Y)\,
p(\bw_*\mid\bx_*,\bX,\bW)\,
p(\bW\mid\bX,Y)\,
d\bw_*\,d\bW,
\end{equation}
again suppressing dependence on the covariance parameters for notational
simplicity. This integral is also generally intractable, motivating
approximate methods for both inference and prediction methods.

Existing approaches include variational inference (DSVI, \cite{damianou2013deep,salimbeni2017doubly}), stochastic-gradient
sampling methods (SGHMC, \cite{DGPHMC2018}), and MCMC-based approaches
(MCMC-VEC, \cite{sauer2023active,sauer2023vecchia}). These methods
make different trade-offs between posterior approximation and
computational cost. DSVI achieves efficient optimization through
inducing-point and variational approximations, SGHMC uses
stochastic-gradient sampling to approximate the posterior over inducing
variables, and MCMC-VEC retains a fully Bayesian treatment under a
Vecchia-approximated model but requires repeated posterior sampling.
In contrast, LaMP-DGP replaces posterior integration over the
training latent states with direct MAP optimization, avoiding both
variational approximation and posterior sampling during latent-state
estimation.

\section{Latent MAP Deep Gaussian Processes}\label{sec:LaMP-DGP}

Our proposed LaMP-DGP model is introduced in this section.  Let $\boldsymbol{\phi}=(\boldsymbol{\theta}_1,\theta_2,g)$ denote the
covariance parameters. Rather than fully integrating over the latent
representation $\bW$ in \eqref{eq:marginal likelihood}, we estimate $\bW$
by maximum a posteriori (MAP) while jointly optimizing the covariance
parameters:
\begin{align}
\label{eq:W MAP}
(\widehat{\bW},\widehat{\boldsymbol{\phi}})
=
\arg\max_{\bW,\boldsymbol{\phi}}
\left\{
\log \underbrace{p(\bW\mid\bX,\boldsymbol{\theta}_1)}
_{\text{first-layer GP prior}}
+
\log \underbrace{p(Y\mid\bW,\theta_2,\tau^2,g)}
_{\text{second-layer likelihood}}
\right\}.
\end{align}
For fixed $\bW$, $\theta_2$, and $g$, the second-layer process variance
$\tau^2$ can be profiled out using its maximum likelihood estimator
\begin{equation}\label{eq:profiletau}
\widehat{\tau}^2
=
\frac{1}{n}
Y^\top
\left(K_{\theta_2}(\bW)+g\mathbf I_n\right)^{-1}
Y.
\end{equation}
Substituting $\widehat{\tau}^2$ into \eqref{eq:W MAP}, the joint estimator
can equivalently be obtained by minimizing the scaled negative log-posterior $\cL_n(\bW,\boldsymbol{\phi})$:
\begin{equation}
\label{eq:objective}
\begin{aligned}
(\widehat{\bW},\widehat{\boldsymbol{\phi}})
&=
\arg\min_{\bW,\boldsymbol{\phi}}
\cL_n(\bW,\boldsymbol{\phi}),
\\
\cL_n(\bW,\boldsymbol{\phi})
&=
-\frac{1}{n}
\left(
\ell_1(\bW,\boldsymbol{\theta}_1)
+
\ell_2(\bW,\theta_2,g)
\right),
\end{aligned}
\end{equation}
where, up to additive constants,
\begin{equation}\label{eq:loglikelihood}
  \begin{aligned}
\ell_1(\bW,\boldsymbol{\theta}_1)
&=
-\frac{1}{2}
\sum_{j=1}^D
\left(
W_{\cdot j}^\top
K_{\theta_{1j}}(\bX)^{-1}
W_{\cdot j}
+
\log\left|K_{\theta_{1j}}(\bX)\right|
\right),
\\
\ell_2(\bW,\theta_2,g)
&=
-\frac{n}{2}
\log\!\left[
Y^\top
\left(K_{\theta_2}(\bW)+g\mathbf I_n\right)^{-1}
Y
\right]
-\frac{1}{2}
\log\left|
K_{\theta_2}(\bW)+g\mathbf I_n
\right|.
  \end{aligned}
\end{equation}

The objective in \eqref{eq:objective} balances the first-layer GP prior,
which favors latent representations that are compatible with the input
space, with the second-layer likelihood, which favors latent configurations
that explain the observed response. A key advantage of this formulation is
that it converts inference for the training latent states into a direct
optimization problem. Moreover, the gradients of
$\cL_n(\bW,\boldsymbol{\phi})$ are available in closed form, enabling
efficient gradient-based optimization without posterior sampling of the
training latent states. The required gradients are derived in the following
subsection.



\subsection{Gradient-Based Optimization}
The quantities optimized in LaMP-DGP are the latent variables $\bW$,
the first- and second-layer length-scale parameters
$\boldsymbol{\theta}_1$ and $\theta_2$, and the nugget $g$; the
second-layer process variance $\tau^2$ is profiled out in closed form as in \eqref{eq:profiletau}. For notational simplicity, define $\Sigma_{1j} := K_{{\theta}_{1j}}(\bX),\;\Sigma_{2} := K_{\theta_2}(\bW) + g \mathbf{I}_n$. The derivations are based on the squared-exponential kernel in \eqref{eq:exkernel}. We suppress the arguments of \(\ell_1(\bW,\boldsymbol{\theta}_1)\), \(\ell_2(\bW,\theta_2,g)\), and \(\cL_n(\bW,\boldsymbol{\phi})\). All partial derivatives are taken
with the remaining arguments held fixed. 

\subsubsection{Gradient with respect to $\bW$}
Since
\(
\cL_n
=
-\frac{1}{n}(\ell_1+\ell_2),
\)
the gradient with respect to $\bW$ decomposes as
\(
n\frac{\partial\cL_n}{\partial\bW}
=
-\frac{\partial\ell_1}{\partial\bW}
-\frac{\partial\ell_2}{\partial\bW}.
\)
For the first-layer contribution,
\(
-\frac{\partial\ell_1}{\partial W_{\cdot j}}
=
\Sigma_{1j}^{-1}W_{\cdot j}
\)
for $j=1,\ldots,D$.
For the second-layer contribution, define
\begin{align}
\mathbf B
:=
-\frac{\partial\ell_2}{\partial\Sigma_2}
=
\frac12
\left(
\Sigma_2^{-1}
-
\frac{
\Sigma_2^{-1}YY^\top\Sigma_2^{-1}
}{
\widehat{\tau}^{2}
}
\right),
\label{eq:l2_Sig}
\end{align}
and let $B_{ik}$ denote the $(i,k)$th element of $\mathbf B$.
Applying the chain rule element-wise gives
\begin{align*}
-\frac{\partial\ell_2}{\partial W_{ij}}
&=
\sum_{k=1}^n
B_{ik}
\frac{\partial[\Sigma_2]_{ik}}{\partial W_{ij}}
+
\sum_{k=1}^n
B_{ki}
\frac{\partial[\Sigma_2]_{ki}}{\partial W_{ij}} =
\frac{4}{\theta_2}
\sum_{k=1}^n
B_{ik}
[K_{\theta_2}(\bW)]_{ik}
(W_{kj}-W_{ij}).
\end{align*}
Define
\(
\mathbf A
:=
\mathbf B\circ K_{\theta_2}(\bW),
\)
where \(\circ\) denotes the Hadamard product. The second-layer
contribution can then be written compactly as
\(
-\frac{\partial \ell_2}{\partial \bW}
=
\frac{4}{\theta_2}
\left\{
\mathbf A\bW
-
\operatorname{diag}(\mathbf A\mathbf 1_n)\bW
\right\}.
\)
Combining the contributions from both layers gives
\[
\frac{\partial\cL_n}{\partial\bW}
=
\frac{1}{n}
\left[
\left(
\Sigma_{11}^{-1}W_{\cdot1},
\ldots,
\Sigma_{1D}^{-1}W_{\cdot D}
\right)
+
\frac{4}{\theta_2}
\left\{
\mathbf A\bW
-
\operatorname{diag}(\mathbf A\mathbf1_n)\bW
\right\}
\right].
\]

\subsubsection{Gradient with respect to hyperparameters}

Since $\theta_2$ appears only in $\ell_2$, its derivative is
\(
n\frac{\partial\cL_n}{\partial\theta_2}
=
-\frac{\partial\ell_2}{\partial\theta_2}
=
\operatorname{tr}\!\left(
\mathbf B^\top
\frac{\partial\Sigma_2}{\partial\theta_2}
\right),
\)
which gives
\[
\frac{\partial\cL_n}{\partial\theta_2}
=
\frac{1}{n}
\sum_{i=1}^n\sum_{k=1}^n
B_{ik}
[K_{\theta_2}(\bW)]_{ik}
\frac{\|\bw_i-\bw_k\|^2}{\theta_2^2}.
\]
Similarly, since
\(
\partial\Sigma_2/\partial g=\mathbf I_n,
\)
the gradient with respect to the nugget is
\(
\frac{\partial\cL_n}{\partial g}
=
\frac{1}{n}
\operatorname{tr}(\mathbf B).
\)


For the first-layer length-scale parameter $\theta_{1j}$, define
\[
\mathbf C_{1j}
:=
\frac12
\left(
\Sigma_{1j}^{-1}
-
\Sigma_{1j}^{-1}
W_{\cdot j}W_{\cdot j}^\top
\Sigma_{1j}^{-1}
\right).
\]
Then,
\(
n\frac{\partial\cL_n}{\partial\theta_{1j}}
=
\operatorname{tr}\!\left(
\mathbf C_{1j}^\top
\frac{\partial\Sigma_{1j}}{\partial\theta_{1j}}
\right)
\), which gives
\[
\frac{\partial\cL_n}{\partial\theta_{1j}}
=
\frac{1}{n}
\sum_{i=1}^n\sum_{k=1}^n
[\mathbf C_{1j}]_{ik}
[K_{\theta_{1j}}(\bX)]_{ik}
\frac{\|\bx_i-\bx_k\|^2}{\theta_{1j}^2}.
\]
The full gradient is obtained by stacking these scalar derivatives,
\(
\frac{\partial \cL_n}{\partial \boldsymbol{\theta}_1}
=
\left(
\frac{\partial\mathcal L_n}{\partial\theta_{11}},
\ldots,
\frac{\partial\mathcal L_n}{\partial\theta_{1D}}
\right)^\top.
\)
In implementation, the covariance parameters are optimized on the log
scale to preserve positivity; thus, for any positive parameter $\alpha$,
the corresponding gradient is obtained from
\(
\frac{\partial\cL_n}{\partial\log\alpha}
=
\alpha\frac{\partial\cL_n}{\partial\alpha}.
\)


\subsubsection{Adam optimizer}

Given the gradients with respect to the latent variables and covariance
parameters, we employ the Adam optimizer \citep{kingma2015adam}, allowing
different optimization hyperparameters for different parameter blocks.
Let
\(
\boldsymbol{\psi}
\in
\left\{
\bW,\boldsymbol{\theta}_1,\theta_2,g
\right\}
\)
denote a generic parameter block, and let
\(
\bh_{\boldsymbol{\psi},t}
=
\nabla_{\boldsymbol{\psi}}
\cL_n\!\left(\boldsymbol{\psi}^{(t)}\right)
\)
denote its gradient at iteration \(t\). Starting from
\(\bm_{\boldsymbol{\psi},0}=\boldsymbol 0\) and
\(\bv_{\boldsymbol{\psi},0}=\boldsymbol 0\), the Adam updates are
\begin{equation}
\begin{aligned}
\bm_{\boldsymbol{\psi},t}
=
\beta_{1,\boldsymbol{\psi}}
\bm_{\boldsymbol{\psi},t-1}
+
(1-\beta_{1,\boldsymbol{\psi}})
\bh_{\boldsymbol{\psi},t},&
\qquad
\bv_{\boldsymbol{\psi},t}
=
\beta_{2,\boldsymbol{\psi}}
\bv_{\boldsymbol{\psi},t-1}
+
(1-\beta_{2,\boldsymbol{\psi}})
\bh_{\boldsymbol{\psi},t}^{\,2},
\\
\boldsymbol{\psi}^{(t+1)}
=
\boldsymbol{\psi}^{(t)}
-
\alpha_{\boldsymbol{\psi},t}
\frac{
\widehat{\bm}_{\boldsymbol{\psi},t}
}{
\sqrt{\widehat{\bv}_{\boldsymbol{\psi},t}}
+
\epsilon_{\boldsymbol{\psi}}
};& \qquad \widehat{\bm}_{\boldsymbol{\psi},t}
=
\frac{\bm_{\boldsymbol{\psi},t}}
{1-\beta_{1,\boldsymbol{\psi}}^t},
\quad
\widehat{\bv}_{\boldsymbol{\psi},t}
=
\frac{\bv_{\boldsymbol{\psi},t}}
{1-\beta_{2,\boldsymbol{\psi}}^t}.
\end{aligned}
\label{eq:adam_update}
\end{equation}
Here, \(\alpha_{\boldsymbol{\psi},t}>0\) is the learning rate,
\(\beta_{1,\boldsymbol{\psi}},\beta_{2,\boldsymbol{\psi}}\in(0,1)\)
are the exponential-decay parameters for the first and second moments,
respectively, and \(\epsilon_{\boldsymbol{\psi}}>0\) is a numerical-stability
constant. All multiplication, division, powers, and square roots in
\eqref{eq:adam_update} are applied element-wise.
The learning rate and Adam hyperparameters may differ across parameter
blocks.

In practice, we randomly partition the observations into training and
validation sets using either a 90\%/10\% or an 80\%/20\% split. The training set is used to optimize the
latent states and covariance hyperparameters, while the validation set is
used to monitor out-of-sample predictive performance and select the best
iterate. Specifically, we retain the estimates from the epoch with the
smallest validation root mean square error (RMSE).

\subsection{Prediction}\label{sec:prediction}

After obtaining the estimates
$\widehat{\bW}$, $\hat{\boldsymbol{\theta}}_1$,
$\hat{\theta}_2$, and $\hat g$ from the gradient-based optimization,
together with the corresponding profiled variance estimate $\hat\tau^2$,
prediction proceeds by conditioning on these fitted quantities.
Specifically, the posterior distribution of the training latent
representation in \eqref{eq:prediction_original} is approximated by a
point mass at its MAP estimate,
\[
p(\bW\mid\bX,Y)
\approx
\delta_{\widehat{\bW}}(\bW).
\]
Consequently,
\begin{equation}\label{eq:prediction_MAP}
p(f_*\mid\bx_*,\bX,Y)
\approx
\int
p(f_*\mid\bw_*,\widehat{\bW},Y)\,
p(\bw_*\mid\bx_*,\bX,\widehat{\bW})\,
d\bw_*,    
\end{equation}
where the estimated covariance parameters are suppressed for notational
simplicity. Thus, LaMP-DGP replaces uncertainty in the training latent
representation by its MAP estimate while retaining and propagating
uncertainty in the test latent representation $\bw_*$. Extensions that account for uncertainty in the training latent states are discussed in Section \ref{sec:conclusion}.

Conditional on the fitted model, prediction therefore proceeds
sequentially through the two GP layers. The first layer provides a
Gaussian predictive distribution for each latent coordinate $w_j^*$,
$j=1,\ldots,D$, and the resulting latent representation $\bw_*$ is then
passed through the second-layer GP. Specifically,
\begin{equation}\label{eq:predictive_w}
\begin{aligned}
w_j^*
&\mid \bx_*,\bX,\widehat{\bW},
\hat\theta_{1j}
\overset{\mathrm{ind}}{\sim}
\mathcal N\!\left(
    \mu_{w_j}(\bx_*),
    \sigma_{w_j}^2(\bx_*)
\right),\quad\text{for}\quad j=1,\ldots,D,\\
f_*
&\mid\bw_*,\widehat{\bW},Y,
\hat\theta_2,\hat\tau^2,\hat g
\sim
\mathcal N\!\left(
    \mu_f(\bw_*),
    \sigma_f^2(\bw_*)
\right).
\end{aligned}
\end{equation}
where
\[
\begin{aligned}
\mu_{w_j}(\bx_*)
&=
\bk_{\hat\theta_{1j}}(\bx_*,\bX)
K_{\hat\theta_{1j}}(\bX)^{-1}
\widehat W_{\cdot j},
\\\sigma_{w_j}^2(\bx_*)
&=
1-
\bk_{\hat\theta_{1j}}(\bx_*,\bX)
K_{\hat\theta_{1j}}(\bX)^{-1}
\bk_{\hat\theta_{1j}}(\bX,\bx_*),
\\
\mu_f(\bw_*)
&=
\bk_{\hat\theta_2}(\bw_*,\widehat{\bW})
\left\{
K_{\hat\theta_2}(\widehat{\bW})
+\hat g\mathbf I_n
\right\}^{-1}Y,
\\
\sigma_f^2(\bw_*)
=
\hat\tau^2
&\left[
1
-
\bk_{\hat\theta_2}(\bw_*,\widehat{\bW})
\left\{
K_{\hat\theta_2}(\widehat{\bW})
+\hat g\mathbf I_n
\right\}^{-1}
\bk_{\hat\theta_2}(\widehat{\bW},\bw_*)
\right].
\end{aligned}
\]
with
\(
\bk_{\hat\theta_{1j}}(\bx_*,\bX)
=
\left(
K_{\hat\theta_{1j}}(\bx_*,\bx_1),
\ldots,
K_{\hat\theta_{1j}}(\bx_*,\bx_n)
\right),
\)
and
$\bk_{\hat\theta_{1j}}(\bX,\bx_*)
=
\bk_{\hat\theta_{1j}}(\bx_*,\bX)^\top$;
the second-layer covariance vectors are defined analogously.

Although the conditional distributions in \eqref{eq:predictive_w} are
Gaussian, the marginal predictive distribution of $f^*$ in
\eqref{eq:prediction_MAP} is generally non-Gaussian because both
$\mu_f(\bw_*)$ and $\sigma_f^2(\bw_*)$ depend nonlinearly on the random
latent representation $\bw_*$. Nevertheless, using the linked-GP moment
calculations of \citet{Kyzyurova2018Berger} and \citet{ming2021linked}, the first two moments of this
predictive distribution can be obtained in closed form. We therefore
approximate the marginal predictive distribution by a Gaussian distribution
with the corresponding mean and variance. Among Gaussian distributions,
this moment-matched approximation minimizes the forward
Kullback--Leibler divergence from the true predictive distribution,
providing a convenient and computationally efficient approximation.

Specifically, following \citet{ming2021linked}, under the
squared-exponential kernel the predictive mean is
\begin{align*}
\mathbb E[f_*\mid\bx_*,\bX,Y]
&=
\mathbb E_{\bw_*\mid\bx_*,\bX,\widehat{\bW}}
\!\left[
\mu_f(\bw_*)
\right]=
\overline{\bk}_*(\bx_*)
\left\{
K_{\hat\theta_2}(\widehat{\bW})
+\hat g\mathbf I_n
\right\}^{-1}Y.
\end{align*}
where
\(
\overline{\bk}_*(\bx_*)
:=
\mathbb E_{\bw_*\mid\bx_*,\bX,\widehat{\bW}}
\left[
\bk_{\hat\theta_2}(\bw_*,\widehat{\bW})
\right].
\)
For each $i=1,\ldots,n$, the $i$th element of this vector has the
closed-form expression:
\begin{align*}
[\overline{\bk}_*(\bx_*)]_i
&=
\mathbb E_{\bw_*\mid\bx_*,\bX,\widehat{\bW}}\!\left[
    \exp\!\left(
        -\frac{\|\bw_*-\widehat{\bw}_i\|^2}
        {\hat\theta_2}
    \right)
\right] \\
&=
\left|
\mathbf I_D
+\frac{2\Sigma_w(\bx_*)}{\hat\theta_2}
\right|^{-1/2}
\exp\!\left[
-(\widehat{\bw}_i-\boldsymbol\mu_w(\bx_*))^\top
(\hat\theta_2\mathbf I_D+2\Sigma_w(\bx_*))^{-1}
(\widehat{\bw}_i-\boldsymbol\mu_w(\bx_*))
\right],
\end{align*}
where $\widehat{\bw}_i^\top$ denotes the $i$th row of
$\widehat{\bW}$,
\(
\boldsymbol\mu_w(\bx_*)
=
\bigl(
\mu_{w_1}(\bx_*),\ldots,\mu_{w_D}(\bx_*)
\bigr)^\top\), and
\(\Sigma_w(\bx_*)
=
\operatorname{diag}\!\left(
\sigma_{w_1}^2(\bx_*),\ldots,
\sigma_{w_D}^2(\bx_*)
\right).
\)
For the predictive variance, the law of total variance gives
\begin{align*}
\operatorname{Var}&(f_*\mid\bx_*,\bX,Y)
=
\mathbb E_{\bw_*\mid\bx_*,\bX,\widehat{\bW}}
\left[
\sigma_f^2(\bw_*)
\right]
+
\operatorname{Var}_{\bw_*\mid\bx_*,\bX,\widehat{\bW}}
\left[
\mu_f(\bw_*)
\right]
\\
&=
\hat\tau^2
\left[
1-
\operatorname{tr}\!\left(
\left\{
K_{\hat\theta_2}(\widehat{\bW})
+\hat g\mathbf I_n
\right\}^{-1}
\mathbf H(\bx_*)
\right)
\right]
\\
&\qquad+
Y^\top
\left\{
K_{\hat\theta_2}(\widehat{\bW})
+\hat g\mathbf I_n
\right\}^{-1}
\left[
\mathbf H(\bx_*)
-
\overline{\bk}_*(\bx_*)^\top
\overline{\bk}_*(\bx_*)
\right]
\left\{
K_{\hat\theta_2}(\widehat{\bW})
+\hat g\mathbf I_n
\right\}^{-1}
Y,
\end{align*}
where
\(
\mathbf H(\bx_*)
:=
\mathbb E_{\bw_*\mid\bx_*,\bX,\widehat{\bW}}
\left[
\bk_{\hat\theta_2}(\widehat{\bW},\bw_*)
\bk_{\hat\theta_2}(\bw_*,\widehat{\bW})
\right],
\)
with its $(i,j)$th element given by
\begin{align*}
[\mathbf H(\bx_*)]_{ij}
&=
\left|
\mathbf I_D+
\frac{4\Sigma_w(\bx_*)}{\hat\theta_2}
\right|^{-1/2}
\exp\!\left(
-\frac{
\|\widehat{\bw}_i-\widehat{\bw}_j\|^2
}{
2\hat\theta_2
}
\right)
\\
&\quad\times
\exp\!\left[
-2
\left(
\boldsymbol\mu_w(\bx_*)
-\frac{\widehat{\bw}_i+\widehat{\bw}_j}{2}
\right)^\top
\left(
\hat\theta_2\mathbf I_D+4\Sigma_w(\bx_*)
\right)^{-1}
\left(
\boldsymbol\mu_w(\bx_*)
-\frac{\widehat{\bw}_i+\widehat{\bw}_j}{2}
\right)
\right].
\end{align*}

These closed-form expressions eliminate the need for Monte Carlo approximation of the predictive mean and variance by analytically integrating over the test latent representation $\bw_*$, leading to efficient prediction. The full derivation can be found in \cite{ming2021linked}.

\section{Mini-Batch Latent MAP Deep Gaussian Processes}
\label{sec:LaMP-DGP SGD}

The gradient-based optimization for LaMP-DGP developed in
Section~\ref{sec:LaMP-DGP} is computationally efficient for moderate sample sizes. For larger datasets, however, each full-data gradient evaluation requires constructing and factorizing dense kernel matrices, which entails $\mathcal{O}(n^2)$ memory and $\mathcal{O}(n^3)$ computational cost. To alleviate this bottleneck, we develop a mini-batch optimization strategy that evaluates the LaMP-DGP objective and its gradients using subsets of the observations. We also study the effect of mini-batch bias on optimization and use the resulting analysis to motivate a progressively increasing mini-batch schedule.

\subsection{Mini-Batch Optimization Algorithm}

The mini-batch procedure for LaMP-DGP is summarized in
Algorithm~\ref{algo:mini-batch LaMP-DGP}. For a subset
$S\subset\{1,\ldots,n\}$ with $|S|=m$, define the mini-batch objective
\[
\cL_S(\bW_S,\boldsymbol{\phi})
=
-\frac{1}{m}
\left\{
\ell_{1,S}(\bW_S,\boldsymbol{\theta}_1)
+
\ell_{2,S}(\bW_S,\theta_2,g)
\right\},
\]
where
\(
\boldsymbol{\phi}
=
(\boldsymbol{\theta}_1,\theta_2,g)
\),
and $\ell_{1,S}$ and $\ell_{2,S}$ are obtained from the corresponding full-data quantities by restricting $\bX$, $Y$, and $\bW$ to the observations indexed by $S$.

At epoch $t$, the training observations are randomly partitioned into mini-batches of size $m_t$. For each mini-batch $S_{t,b}$, the objective $\cL_{S_{t,b}}$ is used to compute gradients with respect to only the corresponding latent rows $\bW_{S_{t,b}}$. The associated Adam moments and update counters are updated for these rows, while the remaining rows of $\bW$ are left unchanged. Because the mini-batches form a partition of the training data, every latent state is updated exactly once per epoch.

Each mini-batch involves covariance matrices of dimension $m_t$, resulting in a dominant dense linear-algebra cost per mini-batch is
$\mathcal O(m_t^3)$ for covariance factorization, with
$\mathcal O(m_t^2)$ memory, compared with $\mathcal O(n^3)$ computation
and $\mathcal O(n^2)$ memory for a full-data update. With approximately
$n/m_t$ mini-batches per epoch, this gives an aggregate factorization
cost of $\mathcal O(nm_t^2)$ per covariance matrix per epoch, reducing
both computational cost and memory requirements relative to full-data
optimization when $m_t\ll n$.

The covariance hyperparameters are updated every $q_\phi$ epochs. At each such update, a mini-batch of fixed size $m_\phi$ is sampled, and the gradients with respect to $\boldsymbol{\phi}$ are evaluated while holding the current latent states fixed. The positive covariance parameters are then updated on the log scale using Adam. In contrast, the latent-state mini-batch size $m_t$ is allowed to vary across epochs. In Section~\ref{sec:analysis}, we motivate a progressively increasing schedule for $m_t$ that reduces the bias of the mini-batch gradients.

Because prediction using the entire training set can be expensive, validation RMSE is evaluated every \(q_{\mathrm{val}}=10\) epochs using a fixed subset of \(m_{\mathrm{pred}}\) training observations to construct predictions for the validation set. This reduced prediction set is used only to monitor validation performance and select the fitted model; it does not affect the optimization updates. Among the validation checkpoints, we retain the latent states and covariance parameters corresponding to the smallest validation RMSE.

Specific choices of the batch sizes, update frequencies, prediction subset size, and optimizer settings used in the numerical studies are provided in Appendix~\ref{appendix:optimizer_settings}.


\begin{algorithm}[tbh]
\caption{Mini-Batch LaMP-DGP}
\label{algo:mini-batch LaMP-DGP}
\setstretch{1.05}
\begin{algorithmic}[1]

\State Initialize $\bW^{(0)}$ and
       $\boldsymbol{\phi}^{(0)}
       =
       (\boldsymbol{\theta}_1^{(0)},\theta_2^{(0)},g^{(0)})$

\State Initialize Adam moments and update counters to zero

\State Set latent mini-batch schedule $\{m_t\}_{t=1}^T$,
       hyperparameter batch size $m_\phi$,
       hyperparameter update frequency $q_\phi$,
       and validation frequency $q_{\mathrm{val}}=10$

\State Select a fixed prediction subset
       $S_{\mathrm{pred}}\subset\{1,\ldots,n\}$,
       with $|S_{\mathrm{pred}}|=m_{\mathrm{pred}}$,
       for validation

\State Set $r_{\mathrm{best}}\gets\infty$

\For{$t=1,\ldots,T$}
    \State Randomly partition $\{1,\ldots,n\}$ into
           $\mathcal P_t=\{S_{t,1},\ldots,S_{t,B_t}\}$,
           where $B_t=\lceil n/m_t\rceil$
    \State Set $\bW^{(t,0)}\gets\bW^{(t-1)}$

    \For{$b=1,\ldots,B_t$}

        \State Compute $\bh_{t,b}^{W}
        =
        \nabla_{\bW_{S_{t,b}}}
        \cL_{S_{t,b}}
        \left(
            \bW_{S_{t,b}}^{(t,b-1)};
            \boldsymbol{\phi}^{(t-1)}
        \right)$

        \State Set
        $\bW^{(t,b)}\gets\bW^{(t,b-1)}$

        \State Update
        $\bW_{S_{t,b}}^{(t,b)}$
        and the corresponding Adam moments using
        $\bh_{t,b}^{W}$

    \EndFor

    \State Set
    $\bW^{(t)}\gets\bW^{(t,B_t)}$

    \If{$t \bmod q_\phi = 0$}
    \State Sample $H_t\subset\{1,\ldots,n\}$ with $|H_t|=m_\phi$
    \State Compute $\bh_t^\phi
    =
    \nabla_{\boldsymbol{\phi}}
    \cL_{H_t}
    \left(
        \bW_{H_t}^{(t)};
        \boldsymbol{\phi}^{(t-1)}
    \right)$
    \State Apply an Adam update on the log scale to obtain
    $\boldsymbol{\phi}^{(t)}$
\Else
    \State Set
    $\boldsymbol{\phi}^{(t)}
    \gets
    \boldsymbol{\phi}^{(t-1)}$
\EndIf

\If{$t \bmod q_{\mathrm{val}} = 0$}
    \State Compute validation RMSE $r_t$ using
           $(\bX_{S_{\mathrm{pred}}},
             Y_{S_{\mathrm{pred}}},
             \bW_{S_{\mathrm{pred}}}^{(t)})$
           for prediction
    \If{$r_t<r_{\mathrm{best}}$}
            \State Set $r_{\mathrm{best}}\gets r_t$ and 
$(\widehat{\bW},\widehat{\boldsymbol{\phi}})
            \gets
            (\bW^{(t)},\boldsymbol{\phi}^{(t)})$
        \EndIf

    \EndIf

\EndFor

\State \Return
$(\widehat{\bW},\widehat{\boldsymbol{\phi}})$

\end{algorithmic}
\end{algorithm}

A key distinction from conventional stochastic optimization is that the
mini-batch gradients in LaMP-DGP are generally biased. Such bias is not
necessarily detrimental: biased gradient
estimators may be preferred when they offer computational advantages,
provided that the bias is sufficiently controlled
\citep{ajalloeian2020convergence,driggs2022biased}.
For additive empirical losses, appropriately sampled and scaled mini-batch
gradients are unbiased estimators of the full gradient. In contrast, the GP
likelihood terms in \eqref{eq:loglikelihood} depend jointly on observations
through covariance matrices and cannot be decomposed into independent
per-observation contributions. A mini-batch uses only the covariance submatrices for the selected observations, omitting interactions with observations outside the batch. Its gradient therefore generally differs systematically from the full-data gradient.

We therefore study how mini-batch gradient bias affects convergence to stationarity of the full-data objective. Our analysis builds on results for stochastic optimization with biased gradients \citep{ajalloeian2020convergence}, mini-batch estimation in GP models \citep{chen2022gaussian}, and nonconvex Adam under unbiased stochastic gradients \citep{JMLR:Adam_nonconvex}. These results do not directly apply to Adam updates based on biased latent-state gradients in a DGP. We therefore impose a uniform bias-decay condition with respect to the mini-batch size and establish a stationarity guarantee under an additional assumption that the expected Adam update direction remains sufficiently aligned with the full-data gradient.

\subsection{Setup for Convergence Analysis}\label{sec:setup}

Our theoretical analysis focuses on the latent-state updates, treating the
covariance hyperparameters $\boldsymbol{\phi}$ as fixed. This isolates the
main challenge of interest, namely optimization of the high-dimensional
latent representation $\bW$. 

Let $\mathcal P=\{S_1,\ldots,S_B\}$ denote a partition of
$\{1,\ldots,n\}$ with target mini-batch size $m$. For each
$S\in\mathcal P$, let $R_S$ denote the row-selection operator such that
$R_S\bW=\bW_S$, so that $R_S^\top$ embeds a batch gradient back into the
corresponding rows of an $n\times D$ matrix. We define the assembled
mini-batch gradient by
\[
G_{n,m}(\bW;\boldsymbol{\phi},\mathcal P)
=
\sum_{S\in\mathcal P}
R_S^\top
\nabla_{\bW_S}
\cL_S(\bW_S;\boldsymbol{\phi}).
\]
Thus, each row of $G_{n,m}$ is obtained from the mini-batch containing
the corresponding observation, whereas the full-data gradient
$\nabla_{\bW}\cL_n(\bW;\boldsymbol{\phi})$ is computed from the complete
covariance structure.

At epoch $t$, the assembled mini-batch gradient becomes
\[
G_t
:=
G_{n,m_t}
\left(
\bW^{(t-1)};
\boldsymbol{\phi},
\mathcal P_t
\right)
=
\sum_{b=1}^{B_t}
R_{S_{t,b}}^\top
\bh_{t,b}^{W}.
\]
Because the sets $S_{t,1},\ldots,S_{t,B_t}$ form a disjoint partition,
each row of $\bW$ is updated exactly once during the epoch. Moreover,
$\cL_{S_{t,b}}$ depends only on the observations and latent states indexed
by $S_{t,b}$, so updates from earlier batches do not affect subsequent batch gradients. With Adam hyperparameters and learning rates shared across latent rows, and with Adam states and counters updated only for the selected rows, the sequential mini-batch updates are therefore row-wise equivalent to an epoch-level Adam update driven by  $G_t$. This representation allows us to treat each epoch as a single stochastic update and compare $G_t$ directly with the full-data gradient.

For a random partition $\mathcal P$ generated with target mini-batch size
$m$, define the mini-batch gradient bias as
\begin{equation*}
b_{n,m}(\bW;\boldsymbol{\phi})
=
\mathbb E_{\mathcal P}
\left[
G_{n,m}(\bW;\boldsymbol{\phi},\mathcal P)
\right]
-
\nabla_{\bW}\cL_n(\bW;\boldsymbol{\phi}),
\end{equation*}
where the expectation is taken over the random partition with $\bW$ and
$\boldsymbol{\phi}$ held fixed.
When $m=n$, the partition consists of the full dataset, so that $b_{n,n}(\bW;\boldsymbol{\phi})=0$.

\subsection{Convergence Analysis}\label{sec:analysis}

We adopt the random partitioning scheme described in
Section~\ref{sec:setup}. At each epoch, the partition is generated
independently of the optimization history. Let
$\{m_t\}_{t\ge1}$ be a prescribed deterministic mini-batch-size schedule
and define
\(
\mathcal F_t
=
\sigma\!\left(
\bW^{(0)},\mathcal P_1,\ldots,\mathcal P_{t-1}
\right).
\)
Thus, $\bW^{(t-1)}$ and the Adam states at the start of epoch $t$ are
$\mathcal F_t$-measurable, while $\mathcal P_t$ is independent of
$\mathcal F_t$.

Additional regularity conditions concerning boundedness, smoothness, lower boundedness, covariance nonsingularity, and stochastic-gradient moments are provided in Appendix~\ref{appendix:conditions}. We state below the key conditions governing mini-batch gradient bias decay and the expected alignment of the Adam update direction with the full-data gradient.

\begin{assumption}
\label{as:batch_bias}
There exist constants $C_b>0$, $\gamma>0$, and an integer $m_0\ge1$, independent of $n$, such that, for every pair of integers $n,m$ satisfying $m_0\le m\le n$,
\[
\sup_{\substack{
\bW\in\mathcal W_n\\
\boldsymbol{\phi}\in\mathcal H
}}
\frac{1}{\sqrt{nD}}
\left\|
b_{n,m}(\bW;\boldsymbol{\phi})
\right\|_F
\le
C_b m^{-\gamma},
\]
where $\mathcal W_n$ and $\mathcal H$ are the latent-state and
hyperparameter spaces, respectively, defined in
Assumption~\ref{as:bound_par}.
\end{assumption}

Assumption~\ref{as:batch_bias} imposes a polynomially decaying upper bound on the root-mean-square bias across latent coordinates, uniformly over $\bW\in\mathcal W_n$ and $\boldsymbol{\phi}\in\mathcal H$. This form is motivated by existing results for mini-batch GP optimization by existing results for mini-batch GP optimization; Specifically, \citet{chen2022gaussian} obtain a statistical error of order \(m^{-1/2}\) under exponential or polynomial kernel eigendecay. Because the present DGP setting additionally involves optimization over latent states that enter the second-layer covariance matrix, we retain a general exponent $\gamma>0$ rather than imposing a specific decay rate.



For the epoch gradient, define \(
b_t
:=
b_{n,m_t}
\left(
\bW^{(t-1)};\boldsymbol{\phi}
\right)=\mathbb E[G_t\mid\mathcal F_t]
-
\nabla_{\bW}\cL_n
\left(
\bW^{(t-1)};\boldsymbol{\phi}
\right)\). We apply the Adam recursion in \eqref{eq:adam_update} with
$\boldsymbol{\psi}=\bW$ and $\bh_{\bW,t}=G_t$. Define the Adam search
direction
\[
\bd_t
=
\frac{\widehat{\bm}_{\bW,t}}
{\sqrt{\widehat{\bv}_{\bW,t}}+\epsilon_{\bW}},
\]
so that
\(
\bW^{(t)}
=
\bW^{(t-1)}
-
\alpha_{\bW,t}\bd_t.
\)
Because $\bd_t$ depends on the current stochastic gradient as well as the
accumulated Adam moments, controlling the bias of $G_t$ alone is not
sufficient to guarantee descent of the full-data objective. We therefore
assume that, in expectation, the Adam direction remains sufficiently
aligned with the full-data gradient, up to an error determined by the
mini-batch bias.

\begin{assumption}
\label{as:adam_alignment}
There exist constants $c_0>0$ and $c_b\ge0$ such that, for every $t$,
\[
\frac{1}{nD}
\mathbb E\!\left[
\left\langle
\nabla_{\bW}\cL_n
\left(
\bW^{(t-1)};\boldsymbol{\phi}
\right),
\bd_t
\right\rangle_F
\middle|\mathcal F_t
\right]
\ge
\frac{c_0}{nD}
\left\|
\nabla_{\bW}\cL_n
\left(
\bW^{(t-1)};\boldsymbol{\phi}
\right)
\right\|_F^2
-
\frac{c_b}{nD}
\|b_t\|_F^2
\]
almost surely.
\end{assumption}

Because the optimization objective is nonconvex, we use the gradient norm as a measure of first-order stationarity and establish the following bound.

\begin{theorem}
\label{thm:biased_adam}

Fix $n$ and $\boldsymbol{\phi}\in\mathcal H$. Suppose that
Assumptions~\ref{as:batch_bias}, \ref{as:adam_alignment}, and \ref{as:bound_par}--\ref{as:bounded_second_moment} hold. For a prescribed horizon $T$, let
$\{\bW^{(t)}\}_{t=0}^{T}$ be generated by the Adam recursion in
\eqref{eq:adam_update} with
\(
\boldsymbol{\psi}=\bW,
\bh_{\bW,t}=G_t\)
and 
\(\alpha_{\bW,t}=\frac{\alpha}{\sqrt T}\) for $t=1,\ldots,T$,
where $\alpha>0$,
$0\le\beta_{1,\bW},\beta_{2,\bW}<1$,
$\epsilon_{\bW}>0$, and
$\bm_{\bW,0}=\bv_{\bW,0}=\boldsymbol 0$.
Suppose also that $m_0\le m_t\le n$ for all $t$. Then
\begin{align}
\min_{0\le t\le T-1}
\mathbb E\left[
\frac{1}{nD}
\left\|
\nabla_{\bW}\cL_n
\left(\bW^{(t)};\boldsymbol{\phi}\right)
\right\|_F^2
\right]
&\le
\frac{C_1}{\sqrt T}
+
\frac{c_b}{c_0T}
\sum_{t=1}^T
\mathbb E\left[
\frac{1}{nD}\|b_t\|_F^2
\right]
\notag\\
&\le
\frac{C_1}{\sqrt T}
+
\frac{C_2}{T}
\sum_{t=1}^T
m_t^{-2\gamma},
\label{eq:convergence_bound}
\end{align}
where  $C_1$ and $C_2$ are independent of $T$.


\end{theorem}

Theorem~\ref{thm:biased_adam} bounds the minimum expected
squared full-data gradient norm among
$\bW^{(0)},\ldots,\bW^{(T-1)}$. The bound separates an
$\mathcal O(T^{-1/2})$ optimization term from a term controlling
the effect of mini-batch gradient bias. In particular, the first bound in
\eqref{eq:convergence_bound} shows that if
\[
\frac1T\sum_{t=1}^{T}
\mathbb E\!\left[
\frac{1}{nD}\|b_t\|_F^2
\right]
\longrightarrow 0,
\]
then
\[
\min_{0\le t\le T-1}
\mathbb E\!\left[
\frac{1}{nD}
\left\|
\nabla_{\bW}\cL_n
(\bW^{(t)};\boldsymbol{\phi})
\right\|_F^2
\right]
\longrightarrow 0.
\]
The proof is provided in Appendix~\ref{appendix: proof}.
The following corollary specifies a sufficient batch-growth schedule
for retaining the $\mathcal O(T^{-1/2})$ stationarity rate.




\begin{corollary}
\label{cor:biased_adam_rate}
Suppose the conditions of Theorem~\ref{thm:biased_adam} hold. Consider a sequence of dataset sizes $n$ and optimization horizons $T$ for which
\(
m_t\ge c_m t^a
\)
for $t=1,\ldots,T$,
where $c_m>0$ and $a\ge1/(4\gamma)$ are fixed.
If $C_1$ and $C_2$ remain bounded uniformly along
this sequence, then
\[
\min_{0\le t\le T-1}
\mathbb E\!\left[\frac{1}{nD}
\left\|
\nabla_{\bW}\cL_n
(\bW^{(t)};\boldsymbol{\phi})
\right\|_F^2
\right]
\leq
\frac{C_1 + 2C_2 c_m^{-2\gamma}}{\sqrt{T}}
=
\mathcal O(T^{-1/2}).
\]
\end{corollary}

\begin{proof}
Since $2a\gamma\ge1/2$,
\[
\frac1T\sum_{t=1}^{T}m_t^{-2\gamma}
\le
\frac{c_m^{-2\gamma}}{T}
\sum_{t=1}^{T}t^{-2a\gamma}
\le
\frac{c_m^{-2\gamma}}{T}
\sum_{t=1}^{T}t^{-1/2}
\leq
\frac{2c_m^{-2\gamma}}{\sqrt{T}}.
\]
Substituting into \eqref{eq:convergence_bound} gives the bound.
\end{proof}

The condition $a\ge1/(4\gamma)$ ensures that the
bias term in \eqref{eq:convergence_bound} decays at
least as fast as the optimization term.
For example, if the Assumption~\ref{as:batch_bias} holds
with $\gamma=1/2$, yielding a bias bound of order
$m_t^{-1/2}$, square-root batch growth
$m_t\asymp t^{1/2}$ is sufficient.
Because $m_t\le n$, unbounded polynomial growth
requires the dataset size to increase with the
optimization horizon. For a fixed dataset, we therefore use the practical capped schedule
\begin{equation}\label{eq:batch_schedule}
m_t
=
\min\left\{
n,\,
m_{\max},\,
\max\left(
m_{\min},
\left\lceil ct^{1/2}\right\rceil
\right)
\right\}.
\end{equation}
where $c>0$ and $m_0\le m_{\min}\le m_{\max}\le n$.
Between the lower and upper limits, this schedule follows
square-root growth, consistent with the sufficient rate in
Corollary~\ref{cor:biased_adam_rate} when $\gamma\ge1/2$.

\subsection{Scalable Approximate Prediction}\label{sec:prediction_minibatch}
Section~\ref{sec:prediction} develops the linked-GP predictor by
analytically propagating uncertainty in the latent representation.
Although this avoids Monte Carlo integration, exact prediction still
requires full-data covariance factorizations and linear solves in both
layers. For large $n$, the resulting $\mathcal O(n^3)$ computation and
$\mathcal O(n^2)$ storage can become prohibitive. We therefore use a
hybrid approximation that combines a Nystr\"om approximation \citep{NIPS2000_nystrom} in the
first layer with local GP prediction in the second layer.


For the first-layer GP, let
$Q=\{q_1,\ldots,q_{n_Q}\}\subset\{1,\ldots,n\}$ index a subset of
$n_Q\ll n$ training inputs, and let
\(
\bX_Q
=
(\bx_{q_1},\ldots,\bx_{q_{n_Q}})^\top.
\)
We approximate the corresponding full covariance matrix by
\[
K_{\hat\theta_{1j}}(\bX)
\approx
K_{\hat\theta_{1j}}(\bX,\bX_Q)
K_{\hat\theta_{1j}}(\bX_Q)^{-1}
K_{\hat\theta_{1j}}(\bX_Q,\bX),
\qquad
j=1,\ldots,D,
\]
where $K_{\hat\theta_{1j}}(\bX,\bX_Q)=K_{\hat\theta_{1j}}(\bX_Q,\bX)^\top$ is the cross-covariance matrix between $\bX$ and $\bX_Q$. The first-layer predictive quantities for a new input $\bx_*$ are then
computed using the resulting Nystr\"om approximation, yielding an
approximate latent predictive distribution with mean
$\hat{\mu}_{w_j}(\bx_*)$. Rather than evaluating the full linked-GP predictor in the second
layer, we use $\hat{\mu}_{w_j}(\bx_*)$ as a plug-in latent input.
We identify its $k$ nearest neighbors among the fitted latent states
$\widehat{\bW}$ and construct the second-layer GP using only these
neighbors and their corresponding responses. Consequently, each
second-layer prediction requires factorization of only a $k\times k$
covariance matrix rather than the full $n\times n$ matrix.  In the numerical studies, we use \(n_Q=1000\) training inputs for the Nystr\"om approximation and \(k=100\) nearest neighbors for the local second-layer GP.

This approximation sacrifices propagation of first-layer predictive
uncertainty through the second layer: conditional on
$\hat{\mu}_{w_j}(\bx_*)$ the local GP retains the usual conditional
predictive variance, but it does not integrate over the predictive
distribution of the latent representation. We use the exact linked-GP
predictor when full-data prediction is computationally feasible and
the scalable approximation above for larger datasets.

\section{Numerical Experiments}\label{sec:numerical}

We evaluate the predictive performance and computational efficiency of
LaMP-DGP through a collection of synthetic benchmark problems. The main
experiments focus on deterministic computer experiments, while
additional experiments with stochastic outputs are reported in
Appendix~\ref{appendix:noise}.

\subsection{Simulation Setup}\label{sec:numericalsetup}

We consider training sample sizes
\(
n\in\{100,500,1000,5000,10000,100000\}
\)
and a fixed test sample size of $n_{\mathrm{test}}=1000$, with 10
independent replications for each experimental configuration. In each
replication, the training and test inputs are generated independently
using randomized Latin hypercube designs \citep{LHD1979}. The training
responses are standardized using the sample mean and standard deviation
of the training data. All experiments are conducted on a
high-performance computing cluster, with each task allocated 8 CPU cores
and 128 GB of memory, and with parallelization enabled where supported.

For the proposed  \texttt{LaMP-DGP}, we use the full-batch optimization procedure
described in Section~\ref{sec:LaMP-DGP} for $n=100$, with
5000 epochs. For $n>100$, we use the mini-batch procedure in
Section~\ref{sec:LaMP-DGP SGD} for 1000 epochs, with the
capped square-root batch-size schedule in \eqref{eq:batch_schedule}
using $m_{\min}=150$, $m_{\max}=500$, and $c=15$. Prediction follows
Section~\ref{sec:prediction} for $n\le1000$ and
Section~\ref{sec:prediction_minibatch} for $n>1000$. For the
deterministic experiments, we fix $g=10^{-4}$. In the noisy experiments
of Appendix~\ref{appendix:noise}, $g$ is estimated from the data.
The learning rates and other Adam hyperparameters are reported in Appendix~\ref{appendix:optimizer_settings}.

We compare \texttt{LaMP-DGP} with three existing DGP approaches: MCMC with a Vecchia approximation
(\texttt{MCMC-VEC}), implemented using \textsf{deepgp}
\citep{booth2024deepgp}; stochastic gradient Hamiltonian Monte Carlo (\texttt{SGHMC}), using the authors' publicly available code \citep{havasi2018sghmccode}; and doubly stochastic variational inference
(DSVI). We evaluate DSVI using two implementations:
\texttt{DSVI-GPyTorch}, implemented in \textsf{GPyTorch}
\citep{gardner2018gpytorch}, and
\texttt{DSVI-GPflux}, implemented in \textsf{GPflux}
\citep{dutordoir2021gpflux}. Detailed implementation and tuning settings
are provided in Appendix~\ref{appendix:optimizer_settings}. For all methods,
the latent dimension is set equal to the input dimension, $D=d$.
For computational feasibility, \texttt{MCMC-VEC} is run with a
shorter MCMC chain than the \textsf{deepgp} default, and it is not
evaluated at $n=100000$ because of its computational cost.

Model performance is evaluated using root mean squared error (RMSE) to
assess predictive mean accuracy and the continuous ranked probability
score (CRPS) \citep{gneiting2007strictly} to assess the predictive
distribution. Lower values of both RMSE and CRPS indicate better
predictive performance.

\subsection{Predictive and Computational Performance}

Figures~\ref{fig:schaffer_park_g4_g6_circuit} and \ref{fig:piston_borehole_wingweight_turb12_oo15} summarize the predictive performance
and computational cost of the five emulators on the ten
benchmark functions, with detailed descriptions of the test functions provided in Appendix~\ref{appendix:function}. Among the benchmark set, the
Schaffer and G-functions are particularly challenging for stationary GP
surrogates because of their rapid local variation and abrupt changes in
response behavior, motivating the use of DGP input warping
\citep{sauer2023vecchia}.

\begin{figure}[!hp]
    \centering
    \includegraphics[width=\linewidth]{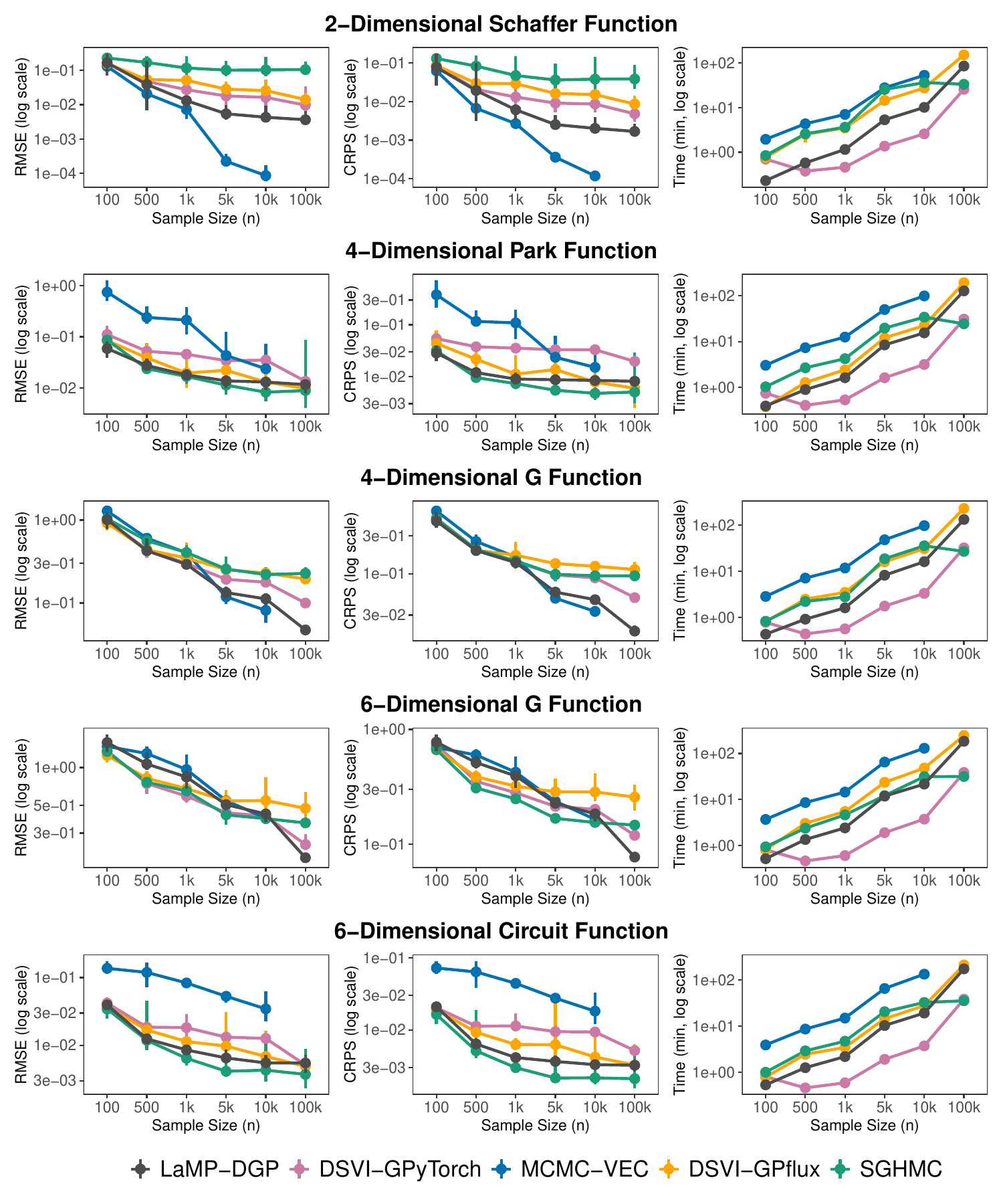}\centering
    \caption{Predictive performance of the five emulators on five benchmark functions (top to bottom): Schaffer function, Park function,  G function, and Circuit function. Points represent means across 10 independent repetitions, and vertical lines indicate the corresponding minimum-to-maximum ranges.}
    \label{fig:schaffer_park_g4_g6_circuit}
\end{figure}

\begin{figure}[htbp]
    \centering
    \includegraphics[width=\linewidth]{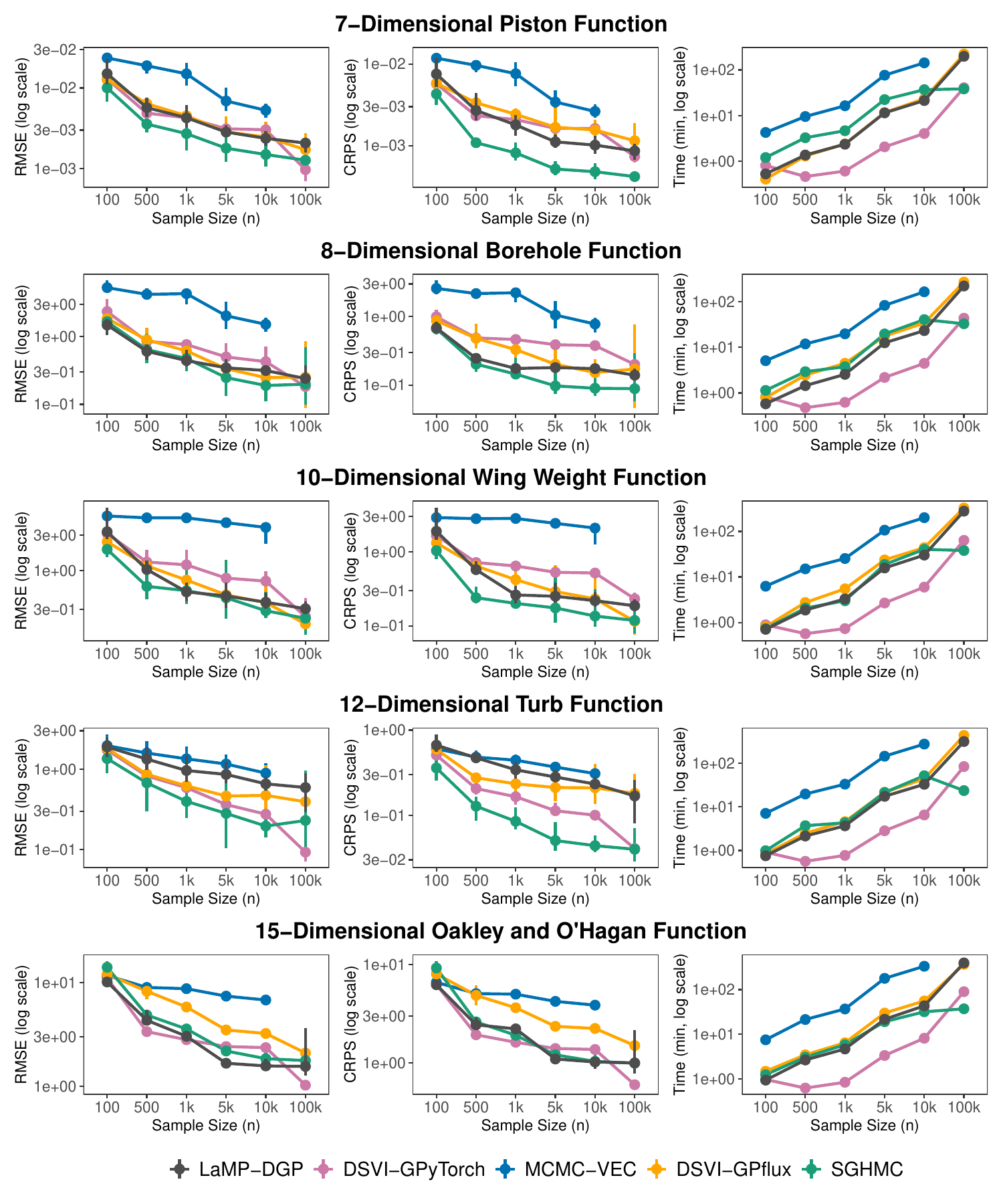}\centering
    \caption{Predictive performance of the five emulators on five benchmark functions (top to bottom): Piston function, Borehole function, Wing weight function, Turb function, and Oakley and O'Hagan function. Points represent means across 10 independent repetitions, and vertical lines indicate the corresponding minimum-to-maximum ranges.}
    \label{fig:piston_borehole_wingweight_turb12_oo15}
\end{figure}

Overall, \texttt{LaMP-DGP} exhibits consistently competitive predictive
performance across functions with varying dimensions, sample sizes, and
response-surface characteristics. Although it does not uniformly attain
the lowest RMSE or CRPS, it performs near the best method across most
benchmarks and avoids substantial deterioration on any problem except
the Turb function. In particular, \texttt{LaMP-DGP} performs
strongly on the G-functions, Park, Circuit, Borehole, Wing Weight, and
Oakley--O'Hagan benchmarks.

The competing methods exhibit more pronounced problem-dependent
behavior. \texttt{MCMC-VEC} performs exceptionally well on the Schaffer
function and is also highly competitive on the G-functions, but is less
competitive on several of the other benchmarks. \texttt{SGHMC} displays an almost complementary pattern: it performs poorly on the Schaffer function and the four-dimensional G-function, notably the two benchmarks highlighted above for their pronounced nonstationary behavior, but is among the strongest methods on several other problems, including Circuit, Piston, Wing Weight, and Turb. The two DGP-DSVI implementations also show noticeably different behavior across benchmarks. \texttt{DSVI-GPyTorch} performs particularly well on the Turb and Oakley--O'Hagan functions at large sample sizes, while \texttt{DSVI-GPflux} is competitive on some problems but exhibits greater variability across benchmarks. In general, the DGP-DSVI methods are less consistently competitive across the benchmark set than \texttt{LaMP-DGP} and \texttt{SGHMC}, although they can perform particularly well on individual problems.  These results
highlight the sensitivity of the competing DGP inference strategies to
the underlying response-surface structure. In contrast,
\texttt{LaMP-DGP} maintains comparatively stable predictive performance
across these different regimes.

The additional noisy experiments in Appendix~\ref{appendix:noise} show
a similar pattern. Under an noise-to-signal ratio (NSR) of $0.01$, \texttt{LaMP-DGP} remains competitive on both the Park and G-function benchmarks. On the Park function, \texttt{SGHMC} is less competitive at small and moderate sample sizes but improves the performance at \(n=100000\), together with \texttt{DSVI-GPyTorch}. On the G-function, \texttt{LaMP-DGP}, \texttt{MCMC-VEC}, and \texttt{SGHMC} are competitive over much of the sample-size range, although \texttt{SGHMC} deteriorates somewhat at the largest sample size.  Overall, the noisy experiments
again show that the relative ranking of the competing methods depends
on the response surface, while \texttt{LaMP-DGP} remains competitive
across both examples.

The computational results reveal a clearer separation among the
methods. \texttt{MCMC-VEC} is consistently the most computationally
expensive approach, with the difference becoming increasingly pronounced
as $n$ grows, and is therefore not evaluated at $n=100000$.
\texttt{DSVI-GPyTorch} is generally the fastest method, particularly at
larger sample sizes. The runtimes of \texttt{LaMP-DGP},
\texttt{SGHMC}, and \texttt{DSVI-GPflux} are broadly comparable across
the benchmark problems. Thus, \texttt{LaMP-DGP} achieves its consistently
competitive predictive performance without the substantially greater
computational cost of \texttt{MCMC-VEC}.

Taken together, the experiments suggest that \texttt{LaMP-DGP} provides
a favorable balance between predictive performance and computational
cost. While individual competitors can be superior for particular
response surfaces, \texttt{LaMP-DGP} is competitive across a broader
range of benchmark structures and sample sizes.

\section{Real-World Computer Experiments}\label{sec:real data}

In this section, we evaluate LaMP-DGP and the competing methods on three
real-world computer experiments, ranging from the relatively small PBX
9501 Gold and Z Machine datasets to the larger GRACE satellite-drag
datasets.

\subsection{GRACE Satellite Drag}

Atmospheric drag is a major source of uncertainty in low-Earth-orbit trajectory prediction. Accurate drag-coefficient estimates are therefore important for orbit determination, collision-risk assessment, and satellite-lifetime estimation \citep{sun2019emulating}. We consider simulations of the Gravity Recovery and Climate Experiment (GRACE) satellite obtained with the Test Particle Monte Carlo (TPMC) simulator \citep{mehta2014modeling}. Under free-molecular-flow conditions, TPMC models interactions between atmospheric particles and the satellite surface to estimate the drag coefficient. The simulator was developed at Los Alamos National Laboratory in \textsf{C} and is accessible through the \textsf{R} interface \texttt{tpm}\footnote{\url{https://bitbucket.org/gramacylab/tpm/src/master/}} \citep{gramacy2020surrogates}.

The GRACE configuration has seven input variables, as listed in Table~2.5 of \citet{gramacy2020surrogates}. We consider three pure atmospheric species: helium (\ce{He}), molecular oxygen (\ce{O2}), and atomic hydrogen (\ce{H}). The scalar response is the simulated drag coefficient. Following \citet{sauer2023vecchia}, we use a $10{,}000$-run
preprocessing sample to fit a Vecchia GP \citep{katzfuss2021general} to obtain vectorized length-scale estimates. We then rescale each input to the unit interval and then pre-scale it using the corresponding estimated length-scale. All emulators are fitted to these pre-scaled inputs. For each species, we select $10{,}000$ simulations for training and reserve $50{,}000$ disjoint simulations for out-of-sample testing. The experiment is repeated independently 10
times. For \texttt{MCMC-VEC}, the nugget is fixed at
$g=10^{-4}$; all remaining method-specific settings follow those
described in Section~\ref{sec:numerical} and
Appendix~\ref{appendix:optimizer_settings}.

Because previous TPMC studies report prediction accuracy using root mean
squared percentage error (RMSPE) \citep{sauer2023vecchia}, we use RMSPE
rather than RMSE for this application. Figure~\ref{fig:grace} summarizes
RMSPE, CRPS, and total computational time across the 10 replications. 

\begin{figure}[h]
    \centering
    \includegraphics[width=\linewidth]{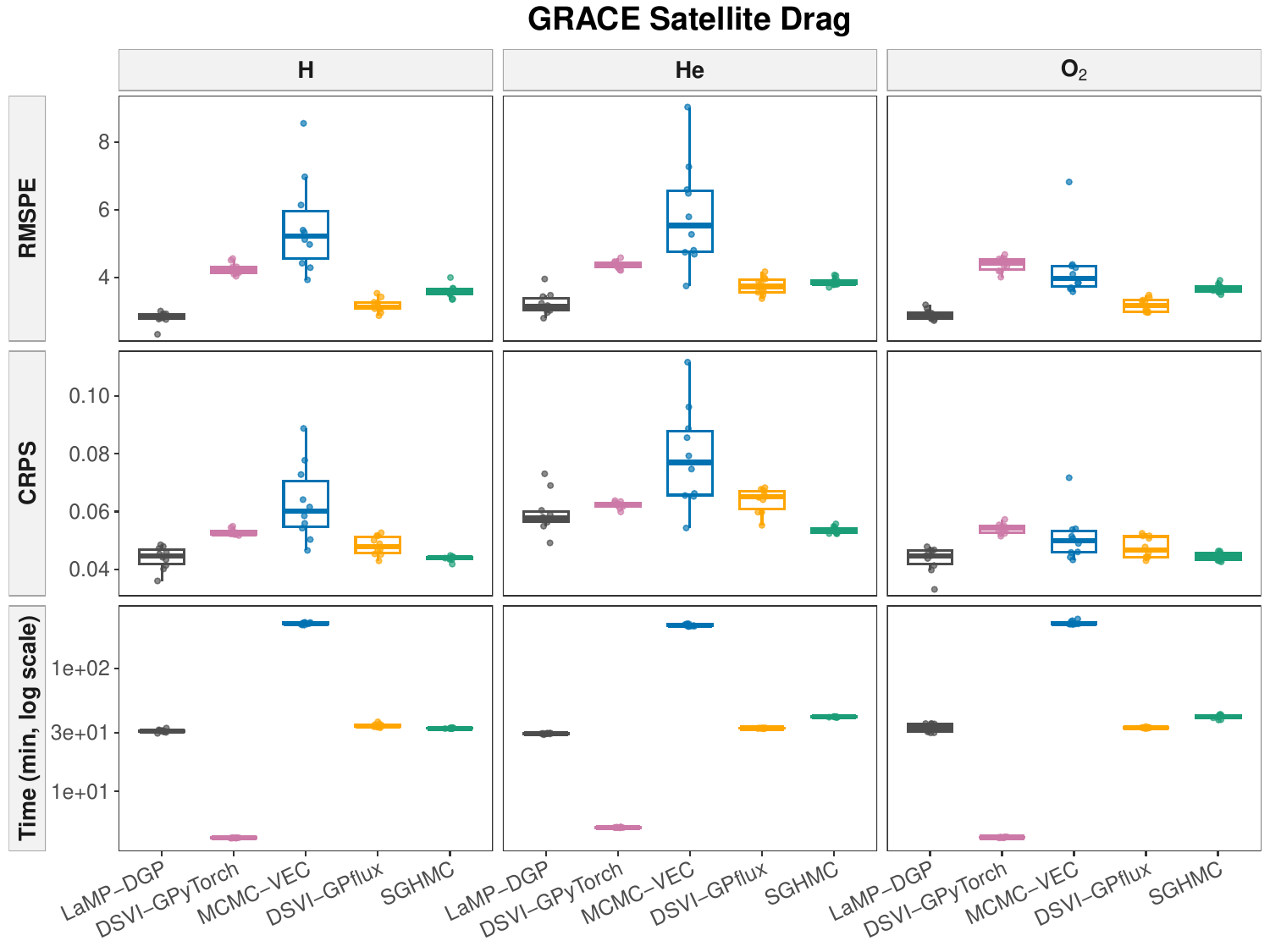}\centering
    \caption{Predictive performance and computational efficiency of five emulators for GRACE satellite-drag simulations under pure-hydrogen (\ce{H}), pure-helium (\ce{He}), and pure-oxygen (\ce{O2}) atmospheric conditions. Boxplots summarize results from 10 independent repetitions.}
    \label{fig:grace}
\end{figure}

\texttt{LaMP-DGP} achieves the lowest median RMSPE for all
three atmospheric species, with relatively small variation across
replications. In terms of CRPS, \texttt{LaMP-DGP} and \texttt{SGHMC} are the strongest methods. The two DGP-DSVI implementations and \texttt{MCMC-VEC} are generally less competitive in predictive performance.  

Computationally, \texttt{DSVI-GPyTorch} is the fastest method, requiring only a few
minutes per replication, while
\texttt{LaMP-DGP}, \texttt{SGHMC}, and \texttt{DSVI-GPflux} require
roughly 30--40 minutes per replication. \texttt{MCMC-VEC} is much more
expensive, requiring over 200 minutes. Thus, on the GRACE application, \texttt{LaMP-DGP} provides
strong predictive accuracy at a substantially lower cost than
\texttt{MCMC-VEC}.

\subsection{Physics-Based Emulation Benchmarks}

We further evaluate the five emulators on the PBX 9501 and Z Machine
benchmarks available through the \textsf{R} package
\texttt{duqling}. Detailed descriptions of the underlying physics and
simulation models are provided in \citet{rumsey2025co} and
\citet{rumsey2025all}. For both benchmarks, we use the ten predefined
training--test splits provided by \texttt{duqling}.

The PBX 9501 dataset consists of rate-stick simulations for a
plastic-bonded high explosive with different jacket materials
\citep{PBX9501,rumsey2025co}. We use the gold-jacket subset, which
contains 500 simulations with six Jones--Wilkins--Lee equation-of-state
parameters as inputs. Following \texttt{duqling}, the scalar response is
the jacket velocity at the fifth PDV location measured
$8\,\mu\mathrm{s}$ after initiation.

For the Z Machine
benchmark, we use the \texttt{Z\_machine\_max\_vel1} dataset
\citep{Zmachine,brown2018}, containing 5000 simulations with six input
parameters. The scalar response is the maximum simulated
tantalum--lithium-fluoride interface velocity in the first experiment.

Figure~\ref{fig:Zmachine_gold} summarizes the predictive performance and
computational cost of the five emulators.  For PBX 9501, the median RMSE
of \texttt{LaMP-DGP} is comparable to those of the DGP-DSVI methods, while it achieves the lowest median CRPS.  For the Z Machine benchmark,  \texttt{DSVI-GPyTorch} and \texttt{SGHMC} attain
slightly lower RMSE than \texttt{LaMP-DGP}, with \texttt{SGHMC}
achieving the lowest CRPS. 
Nevertheless, \texttt{LaMP-DGP} remains
competitive in both metrics and shows relatively stable performance
across the ten data splits.

Computationally, \texttt{DSVI-GPyTorch} is the fastest method, while
\texttt{LaMP-DGP} is comparable to or faster than
\texttt{DSVI-GPflux} and faster than both \texttt{SGHMC} and
\texttt{MCMC-VEC}. Overall, these physics-based benchmarks reinforce
the simulation results in Section~\ref{sec:numerical}: \texttt{LaMP-DGP}
maintains competitive predictive performance across different problems
while requiring moderate computational cost.

\begin{figure}
    \centering
    \includegraphics[width=\linewidth]{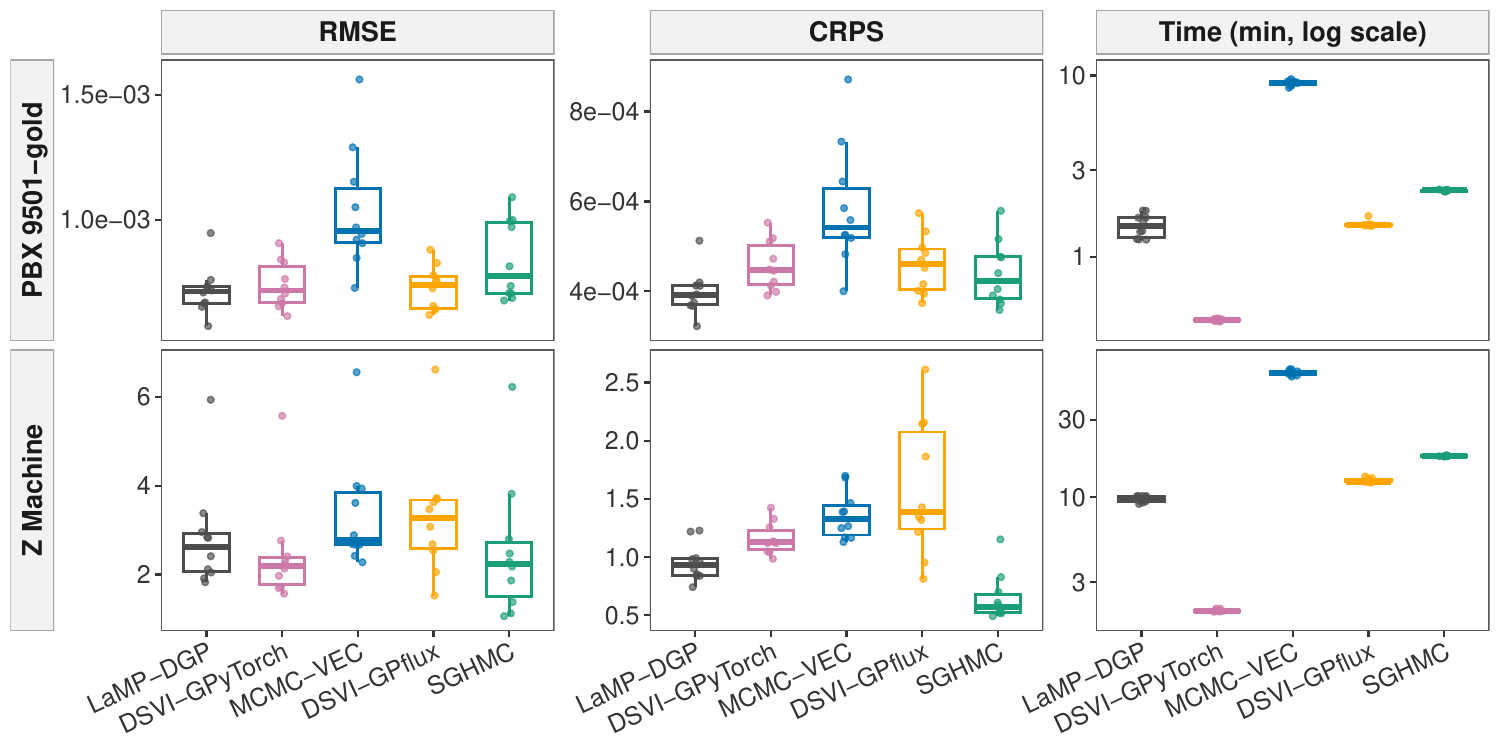}\centering
\caption{Predictive performance of the five emulators for PBX 9501
simulations with a gold jacket (top) and Z Machine simulations
(bottom). Boxplots summarize results across 10 independent
replications.}
    \label{fig:Zmachine_gold}
\end{figure}

\section{Conclusion}\label{sec:conclusion}
Deep Gaussian processes provide flexible surrogate models for complex, nonstationary response surfaces, but inference over their latent representations can be computationally demanding. We proposed LaMP-DGP, which replaces posterior integration over the latent states with direct optimization of a latent MAP objective, turning latent-state inference into a regularized optimization problem that balances fit to the observed responses with the first-layer GP prior. We derived analytic gradients to facilitate efficient optimization, and used linked-GP prediction to efficiently obtain predictive quantities without Monte Carlo integration.

To accommodate larger datasets, we develop a mini-batch optimization
procedure and establish a stationarity bound that accounts for
mini-batch gradient bias, motivating the growing batch-size schedule
used in practice. Across synthetic benchmarks and real-world computer
experiments, LaMP-DGP achieves consistently competitive
predictive performance while being substantially faster than
MCMC-VEC. Variational DGP methods can be faster at large
sample sizes, but their predictive performance is more problem
dependent. In summary, LaMP-DGP provides a favorable balance
between predictive performance and computational efficiency
across a broad range of problems.

Several directions merit further investigation. Although deeper architectures can provide richer hierarchical latent representations, our preliminary experiments with three-layer DGPs showed only modest improvements in predictive accuracy, with occasional deterioration and greater computational cost. Understanding when additional
depth is beneficial, and how to regularize or adaptively select the depth, therefore remains an interesting direction.

Another important direction for future work is to account for uncertainty around the estimated latent configuration. The current MAP formulation uses a single optimized latent representation, which avoids the computational burden of posterior sampling but does not reflect uncertainty in the training latent states. Developing scalable approximations that recover some of this uncertainty while retaining the computational advantages of LaMP-DGP would provide a natural extension of the proposed framework. A promising direction is to develop a scalable Laplace approximation
around the latent MAP solution, using local curvature information to
recover uncertainty in the training latent states without requiring
full posterior sampling.




\vspace{0.2cm}
\noindent\textbf{Supplemental online material.} The supplemental material provides theoretical proofs, additional implementation details, and supporting figures for the numerical studies. The \texttt{LaMPDGP} \textsf{R} package and reproducibility code are also provided.

\if1\blind
\vspace{0.2cm}
\noindent\textbf{Funding.} The authors gratefully acknowledge funding from NSF DMS 2338018.
\fi

\vspace{0.2cm}
\noindent\textbf{Disclosure statement.} The authors report no conflicts of interest.

\vspace{0.2cm}
\noindent\textbf{Data availability statement.}
The data used in this study are either generated from the benchmark functions described in the article or obtained from publicly available sources cited therein. Reproducibility code and supporting materials are provided in the supplementary material.

\vspace{0.2cm}
\noindent\textbf{Declaration of generative AI use.} The authors used OpenAI ChatGPT (GPT-5.6 Sol) to assist with language polishing and software coding. All AI-assisted text and code were reviewed and verified by the authors for accuracy and originality. The authors have reviewed the applicable terms of use and confirm the tool's suitability for publication, and take full responsibility for the integrity of the manuscript, including the accuracy of its references.

\bibliography{ref}

@article{driggs2022biased,
  title={On Biased Stochastic Gradient Estimation},
  author={Driggs, Derek and Liang, Jingwei and Sch{\"o}nlieb, Carola-Bibiane},
  journal={Journal of Machine Learning Research},
  volume={23},
  number={24},
  pages={1--43},
  year={2022}
}

@article{booth2023nonstationary,
  title={Nonstationary {G}aussian process surrogates},
  author={Booth, Annie S and Cooper, Andrew and Gramacy, Robert B},
  journal={arXiv preprint arXiv:2305.19242},
  year={2023}
}

@article{sung2020generalized,
author = {Chih-Li Sung and Ying Hung and William Rittase and Cheng Zhu and C. F. Jeff Wu},
title = {A Generalized {G}aussian Process Model for Computer Experiments With Binary Time Series},
journal = {Journal of the American Statistical Association},
volume = {115},
number = {530},
pages = {945--956},
year = {2020},
publisher = {Taylor \& Francis},
doi = {10.1080/01621459.2019.1604361}
}

@article{chen2022gaussian,
  title={Gaussian process parameter estimation using mini-batch stochastic gradient descent: convergence guarantees and empirical benefits},
  author={Chen, Hao and Zheng, Lili and Al Kontar, Raed and Raskutti, Garvesh},
  journal={Journal of Machine Learning Research},
  volume={23},
  number={227},
  pages={1--59},
  year={2022}
}

@article{rumsey2025all,
  title={All Emulators are Wrong, Many are Useful, and Some are More Useful Than Others: A Reproducible Comparison of Computer Model Surrogates},
  author={Rumsey, Kellin N and Gibson, Graham C and Francom, Devin and Morris, Reid},
  journal={arXiv preprint arXiv:2512.09060},
  year={2025}
}

@article{Ming03042023,
author = {Deyu Ming and Daniel Williamson and Serge Guillas},
title = {Deep {G}aussian Process Emulation using Stochastic Imputation},
journal = {Technometrics},
volume = {65},
number = {2},
pages = {150--161},
year = {2023},
publisher = {ASA Website},
doi = {10.1080/00401706.2022.2124311},
eprint = {https://doi.org/10.1080/00401706.2022.2124311}
}

@article{JMLR:Adam_nonconvex,
  author  = {Chen, Congliang and Shen, Li and Zou, Fangyu and Liu, Wei},
  title   = {Towards Practical {A}dam: Non-Convexity, Convergence Theory, and Mini-Batch Acceleration},
  journal = {Journal of Machine Learning Research},
  year    = {2022},
  volume  = {23},
  number  = {229},
  pages   = {1--47},
  url     = {http://jmlr.org/papers/v23/20-1438.html}
}

@article{sauer2023vecchia,
  title={Vecchia-approximated deep {G}aussian processes for computer experiments},
  author={Sauer, Annie and Cooper, Andrew and Gramacy, Robert B},
  journal={Journal of Computational and Graphical Statistics},
  volume={32},
  number={3},
  pages={824--837},
  year={2023},
  publisher={Taylor \& Francis},
  doi = {10.1080/10618600.2022.2129662}
}

@inproceedings{salimbeni2017doubly,
 author = {Salimbeni, Hugh and Deisenroth, Marc},
 booktitle = {Advances in Neural Information Processing Systems},
 publisher = {Curran Associates, Inc.},
 title = {Doubly Stochastic Variational Inference for Deep {G}aussian Processes},
 url = {https://proceedings.neurips.cc/paper_files/paper/2017/file/8208974663db80265e9bfe7b222dcb18-Paper.pdf},
 volume = {30},
 year = {2017},
 pages = {4588--4599}
}

@inproceedings{kingma2015adam,
  title     = {Adam: A Method for Stochastic Optimization},
  author    = {Kingma, Diederik P. and Ba, Jimmy},
  booktitle = {International Conference on Learning Representations},
  year      = {2015},
  url       = {https://arxiv.org/abs/1412.6980}
}

@article{ajalloeian2020convergence,
  title={On the convergence of {SGD} with biased gradients},
  author={Ajalloeian, Ahmad and Stich, Sebastian U},
  journal={arXiv preprint arXiv:2008.00051},
  year={2020}
}

@article{mak2018efficient,
  title={An efficient surrogate model for emulation and physics extraction of large eddy simulations},
  author={Mak, Simon and Sung, Chih-Li and Wang, Xingjian and Yeh, Shiang-Ting and Chang, Yu-Hung and Joseph, V Roshan and Yang, Vigor and Wu, C F Jeff},
  journal={Journal of the American Statistical Association},
  volume={113},
  number={524},
  pages={1443--1456},
  year={2018},
  publisher={Taylor \& Francis},
  doi = {10.1080/01621459.2017.1409123}
}

@InProceedings{damianou2013deep,
  title = 	 {Deep {G}aussian Processes},
  author = 	 {Damianou, Andreas and Lawrence, Neil D.},
  booktitle = 	 {Proceedings of the Sixteenth International Conference on Artificial Intelligence and Statistics},
  pages = 	 {207--215},
  year = 	 {2013},
  volume = 	 {31},
  series = 	 {Proceedings of Machine Learning Research},
  publisher =    {PMLR}
}

@book{rasmussen2006gaussian,
  title={Gaussian Processes for Machine Learning},
  author={Rasmussen, Carl Edward and Williams, Christopher K. I.},
  year={2006},
  publisher={The MIT Press},
  isbn={026218253X},
  url={http://www.GaussianProcess.org/gpml},
  doi = {10.7551/mitpress/3206.001.0001}
}

@article{sauer2023active,
  title={Active learning for deep {G}aussian process surrogates},
  author={Sauer, Annie and Gramacy, Robert B and Higdon, David},
  journal={Technometrics},
  volume={65},
  number={1},
  pages={4--18},
  year={2023},
  publisher={Taylor \& Francis},
  doi = {10.1080/00401706.2021.2008505}
}

@article{katzfuss2021general,
  title={A general framework for {V}ecchia approximations of {G}aussian processes},
  author={Katzfuss, Matthias and Guinness, Joseph},
  journal={Statistical Science},
  volume={36},
  number={1},
  pages={124--141},
  year={2021},
  publisher={JSTOR},
  doi = {10.1214/19-STS755}
}

@book{gramacy2020surrogates,
  title = {Surrogates: {G}aussian Process Modeling, Design and {O}ptimization for the Applied Sciences},
  author = {Robert B. Gramacy},
  publisher = {Chapman Hall/CRC},
  year = {2020},
  doi = {10.1201/9780367815493},
}

@article{blei2017variational,
  title={Variational inference: A review for statisticians},
  author = {David M. Blei and Alp Kucukelbir and Jon D. McAuliffe},
  journal = {Journal of the American Statistical Association},
  volume = {112},
  number = {518},
  pages = {859--877},
  year = {2017},
  publisher = {Taylor \& Francis},
  doi = {10.1080/01621459.2017.1285773}
}

@article{MORITA2022110788,
title = {Applying {B}ayesian optimization with {G}aussian process regression to computational fluid dynamics problems},
journal = {Journal of Computational Physics},
volume = {449},
pages = {110788},
year = {2022},
issn = {0021-9991},
doi = {https://doi.org/10.1016/j.jcp.2021.110788},
url = {https://www.sciencedirect.com/science/article/pii/S0021999121006835},
author = {Y. Morita and S. Rezaeiravesh and N. Tabatabaei and R. Vinuesa and K. Fukagata and P. Schlatter}
}

@inproceedings{DGPHMC2018,
 author = {Havasi, Marton and Hern\'{a}ndez-Lobato, Jos\'{e} Miguel and Murillo-Fuentes, Juan Jos\'{e}},
 booktitle = {Advances in Neural Information Processing Systems},
 pages = {7506--7516},
 publisher = {Curran Associates, Inc.},
 title = {Inference in Deep {G}aussian Processes using Stochastic Gradient {H}amiltonian {M}onte {C}arlo},
 url = {https://proceedings.neurips.cc/paper_files/paper/2018/file/4172f3101212a2009c74b547b6ddf935-Paper.pdf},
 volume = {31},
 year = {2018}
}

@article{ming2021linked,
  title={Linked {G}aussian process emulation for systems of computer models using {M}at{\'e}rn kernels and adaptive design},
  author={Ming, Deyu and Guillas, Serge},
  journal={SIAM/ASA Journal on Uncertainty Quantification},
  volume={9},
  number={4},
  pages={1615--1642},
  year={2021},
  publisher={SIAM},
  doi = {10.1137/20M1323771}
}

@article{Kyzyurova2018Berger,
author = {Kyzyurova, Ksenia N. and Berger, James O. and Wolpert, Robert L.},
title = {Coupling Computer Models through Linking Their Statistical Emulators},
journal = {SIAM/ASA Journal on Uncertainty Quantification},
volume = {6},
number = {3},
pages = {1151-1171},
year = {2018},
doi = {10.1137/17M1157702}
}

@inproceedings{gardner2018gpytorch,
 author = {Gardner, Jacob and Pleiss, Geoff and Weinberger, Kilian and Bindel, David and Wilson, Andrew G},
 booktitle = {Advances in Neural Information Processing Systems},
 pages = {},
 publisher = {Curran Associates, Inc.},
title     = {{GPyTorch}: Blackbox Matrix-Matrix {G}aussian Process Inference with {GPU} Acceleration},
 url = {https://proceedings.neurips.cc/paper_files/paper/2018/file/27e8e17134dd7083b050476733207ea1-Paper.pdf},
 volume = {31},
 year = {2018}
}

@inproceedings{NIPS2000_nystrom,
 author = {Williams, Christopher and Seeger, Matthias},
 booktitle = {Advances in Neural Information Processing Systems},
 pages = {},
 publisher = {MIT Press},
 title = {Using the {N}ystr\"{o}m Method to Speed Up Kernel Machines},
 url = {https://proceedings.neurips.cc/paper_files/paper/2000/file/19de10adbaa1b2ee13f77f679fa1483a-Paper.pdf},
 volume = {13},
 year = {2000}
}

@Manual{booth2024deepgp,
    title = {deepgp: Bayesian Deep Gaussian Processes using MCMC},
    author = {Annie S. Booth},
    year = {2026},
    note = {R package version 1.2.1},
    url = {https://CRAN.R-project.org/package=deepgp},
    doi = {10.32614/CRAN.package.deepgp},
  }

@article{gneiting2007strictly,
  title={Strictly proper scoring rules, prediction, and estimation},
  author={Gneiting, Tilmann and Raftery, Adrian E},
  journal={Journal of the American Statistical Association},
  volume={102},
  number={477},
  pages={359--378},
  year={2007},
  publisher={Taylor \& Francis},
  doi = {10.1198/016214506000001437}
}

@article{LHD1979,
author = {M. D. McKay and R. J. Beckman and W. J. Conover},
title = {A Comparison of Three Methods for Selecting Values of Input Variables in the Analysis of Output from a Computer Code},
journal = {Technometrics},
volume = {21},
number = {2},
pages = {239--245},
year = {1979},
publisher = {Taylor \& Francis},
doi = {10.1080/00401706.1979.10489755}
}

@phdthesis{park1991tuning,
  author       = {Park, Jeong Soo},
  title        = {Tuning Complex Computer Codes to Data and Optimal Designs},
  school       = {University of Illinois at Urbana-Champaign},
  year         = {1991}
}

@article{ben2007modeling,
  title={Modeling data from computer experiments: an empirical comparison of kriging with {MARS} and projection pursuit regression},
  author = {Einat Neumann Ben-Ari and David M. Steinberg},
journal = {Quality Engineering},
volume = {19},
number = {4},
pages = {327--338},
year = {2007},
publisher = {Taylor \& Francis},
doi = {10.1080/08982110701580930}
}

@article{morris1993bayesian,
  title={Bayesian design and analysis of computer experiments: use of derivatives in surface prediction},
  author={Morris, Max D and Mitchell, Toby J and Ylvisaker, Donald},
  journal={Technometrics},
  volume={35},
  number={3},
  pages={243--255},
  year={1993},
  publisher={Taylor \& Francis},
  doi = {10.1080/00401706.1993.10485320}
}

@book{forrester2008engineering,
  title     = {Engineering Design via Surrogate Modelling: A Practical Guide},
  author    = {Forrester, Alexander I. J. and S\'{o}bester, Andr\'{a}s and Keane, Andy J.},
  year      = {2008},
  publisher = {John Wiley \& Sons},
  isbn      = {978-0-470-06068-1},
  doi       = {10.1002/9780470770801}
}

@article{JMLR:v6:lawrence05a,
  author  = {Neil D. Lawrence},
  title   = {Probabilistic Non-linear Principal Component Analysis with {G}aussian Process Latent Variable Models},
  journal = {Journal of Machine Learning Research},
  year    = {2005},
  volume  = {6},
  number  = {60},
  pages   = {1783--1816},
  url     = {http://jmlr.org/papers/v6/lawrence05a.html}
}

@inproceedings{Lawrence2003gaussian,
  title     = {Gaussian Process Latent Variable Models for Visualisation of High Dimensional Data},
  author    = {Neil D. Lawrence},
  booktitle = {Advances in Neural Information Processing Systems},
  volume    = {16},
  pages     = {329--336},
  year      = {2003},
  publisher = {MIT Press},
  url       = {https://proceedings.neurips.cc/paper/2003/hash/9657c1fffd38824e5ab0472e022e577e-Abstract.html}
}

@InProceedings{lawrence2007hierarchical,
  title = 	 {Hierarchical {G}aussian Process Latent Variable Models},
  author = 	 {Lawrence, Neil D. and Moore, Andrew J.},
  booktitle = 	 {Proceedings of the International Conference in Machine Learning},
  pages = 	 {481--488},
  year = 	 {2007},
  volume = 	 {24},
  publisher =    {Omnipress},
  doi = {10.1145/1273496.1273557}
}

@inproceedings{NIPS2005_ccd45007,
  author    = {Wang, Jack M. and Fleet, David J. and Hertzmann, Aaron},
  title     = {Gaussian Process Dynamical Models},
  booktitle = {Advances in Neural Information Processing Systems},
  publisher = {MIT Press},
  volume    = {18},
  year      = {2005},
  pages     = {1441--1448},
}

@article{dutordoir2021gpflux,
  title     = {{GPflux}: A Library for Deep {G}aussian Processes},
  author    = {Dutordoir, Vincent and Salimbeni, Hugh and Hambro, Eric and McLeod, John and Leibfried, Felix and Artemev, Artem and van der Wilk, Mark and Hensman, James and Deisenroth, Marc P. and John, {ST}},
  journal   = {arXiv preprint arXiv:2104.05674},
  year      = {2021},
  doi = {10.48550/arXiv.2104.05674}
}

@article{rumsey2025co,
  title={Co-active subspace methods for the joint analysis of adjacent computer models},
  author={Rumsey, Kellin N and Hardy, Zachary K and Ahrens, Cory and Vander Wiel, Scott},
  journal={Technometrics},
  volume={67},
  number={1},
  pages={133--146},
  year={2025},
  publisher={Taylor \& Francis},
  doi = {10.1080/00401706.2024.2407314}
}

@misc{PBX9501,
  author = {Rumsey, Kellin N. and Hardy, Zachary K. and Ahrens, Cory and Vander Wiel, Scott},
  publisher = {Harvard Dataverse},
  title     = {{PBX9501}},
  year      = {2024},
  version   = {V1},
  doi       = {10.7910/DVN/YBU6E9},
  url       = {https://doi.org/10.7910/DVN/YBU6E9},
  howpublished = {Harvard Dataverse},
  note = {Version V1. \url{https://doi.org/10.7910/DVN/YBU6E9}},
}

@misc{Zmachine,
  author    = {Brown, Justin and Alexander, C. S. and Asay, J. R. and Dolan, D. H. and Belof, J. L.},
  publisher = {Harvard Dataverse},
  title     = {{Tantalum flow strength ({Z}-machine) simulations and experiments}},
  year      = {2024},
  version   = {V2},
  doi       = {10.7910/DVN/ZCQ0YX},
  url       = {https://doi.org/10.7910/DVN/ZCQ0YX},
  howpublished = {Harvard Dataverse},
  note = {Version V2. \url{https://doi.org/10.7910/DVN/ZCQ0YX}},
}

@article{brown2018,
author = {Brown, J. L. and Hund, L. B.},
title = {Estimating material properties under extreme conditions by using {B}ayesian model calibration with functional outputs},
journal = {Journal of the Royal Statistical Society: Series C (Applied Statistics)},
volume = {67},
number = {4},
pages = {1023-1045},
doi = {https://doi.org/10.1111/rssc.12273},
url = {https://rss.onlinelibrary.wiley.com/doi/abs/10.1111/rssc.12273},
year = {2018}
}

@article{sun2019emulating,
  title={Emulating satellite drag from large simulation experiments},
  author={Sun, Furong and Gramacy, Robert B and Haaland, Benjamin and Lawrence, Earl and Walker, Andrew},
  journal={SIAM/ASA Journal on Uncertainty Quantification},
  volume={7},
  number={2},
  pages={720--759},
  year={2019},
  publisher={SIAM},
  doi = {10.1137/18M1170157}
}

@misc{simulationlib, 
  author       = {Surjanovic, S. and Bingham, D.}, 
  title        = {Virtual Library of Simulation Experiments: Test Functions and Datasets}, 
  year         = {2013},
  howpublished = {Retrieved September 17, 2026, from \url{http://www.sfu.ca/~ssurjano}}
}

@article{marrel2009calculations,
  title={Calculations of {S}obol indices for the {G}aussian process metamodel},
  author={Marrel, Amandine and Iooss, Bertrand and Laurent, Beatrice and Roustant, Olivier},
  journal={Reliability Engineering \& System Safety},
  volume={94},
  number={3},
  pages={742--751},
  year={2009},
  publisher={Elsevier},
  doi = {10.1016/j.ress.2008.07.008}
}

@article{oakley2004probabilistic,
  title={Probabilistic sensitivity analysis of complex models: a {B}ayesian approach},
  author={Oakley, Jeremy E and O'Hagan, Anthony},
  journal={Journal of the Royal Statistical Society Series B: Statistical Methodology},
  volume={66},
  number={3},
  pages={751--769},
  year={2004},
  publisher={Oxford University Press},
  doi = {10.1111/j.1467-9868.2004.05304.x}
}

@article{mehta2014modeling,
  title={Modeling satellite drag coefficients with response surfaces},
  author={Mehta, Piyush M and Walker, Andrew and Lawrence, Earl and Linares, Richard and Higdon, David and Koller, Josef},
  journal={Advances in Space Research},
  volume={54},
  number={8},
  pages={1590--1607},
  year={2014},
  publisher={Elsevier},
  doi = {10.1016/j.asr.2014.06.033}
}

@misc{havasi2018sghmccode,
  author = {Havasi, Marton and
            Hern{\'a}ndez-Lobato, Jos{\'e} Miguel and
            Murillo-Fuentes, Juan Jos{\'e}},
  title = {{sghmc\_dgp}},
  year = {2018},
  howpublished = {GitHub repository},
  url = {https://github.com/cambridge-mlg/sghmc_dgp},
  note = {Commit \texttt{eecd078261def51bd6fda73ebd2b356edf03e17e}}
}

\newpage
\setcounter{page}{1}
\appendix

\noindent{\Huge\bfseries Appendix}

\section{Theoretical Conditions} \label{appendix:conditions}

For every nonempty $S\subseteq\{1,\ldots,n\}$, define
\[
\Sigma_{1j,S}
:=
K_{\theta_{1j}}(\bX_S)+\varepsilon\mathbf I_{|S|},
\qquad
\Sigma_{2,S}
:=
K_{\theta_2}(\bW_S)+g\mathbf I_{|S|},
\]
where $\varepsilon>0$ is a fixed numerical jitter used for numerical
stability. Since the kernel matrices are positive semidefinite,
$\varepsilon>0$, and $g>0$, all full-data and mini-batch covariance
matrices are positive definite and hence nonsingular.

\begin{assumption}
\label{as:bound_par}
For \(i=1,\ldots,n\), let \(\mathbf x_i\) denote the \(i\)-th row of \(\bX\), and let \(Y_i\) denote the corresponding response. There exist finite constants \(B_X,B_Y>0\) such that
\[
\max_{1\le i\le n}\|\mathbf x_i\|_2\le B_X,
\qquad
\max_{1\le i\le n}|Y_i|\le B_Y.
\]
There exists \(B_W>0\) such that the latent iterates satisfy
\[
\bW^{(t)}\in\mathcal W_n
\quad\text{for every }t\ge0,
\]
where
\[
\mathcal W_n
=
\left\{
\bW\in\mathbb R^{n\times D}:
\max_{1\le i\le n}\|\mathbf w_i\|_2\le B_W
\right\},
\]
and \(\bw_i\) denotes the \(i\)th row of \(\bW\). The fixed hyperparameters belong to a compact set
$\mathcal H$. There exist finite positive constants
$\theta_{\min},\theta_{\max},g_{\min},g_{\max}$ such that,
for every $\boldsymbol{\phi}\in\mathcal H$ and every
component $j$,
\[
0<\theta_{\min}
\le \theta_{1j},\theta_2
\le \theta_{\max},
\qquad
0<g_{\min}\le g\le g_{\max}.
\]
Finally, there exist constants
\(0<\tau_{\min}^2\le\tau_{\max}^2<\infty\) such that
\[
\tau_{\min}^2
\le
\widehat{\tau}_S^2
\le
\tau_{\max}^2,
\qquad
\widehat{\tau}_S^2 = \frac{1}{|S|}Y^\top_S \Sigma^{-1}_{2, S}Y_S
\]
uniformly over all admissible parameter values and all full or mini-batch index sets \(S\) considered by the algorithm.

\end{assumption}

\begin{assumption}
\label{as:smoothness}
For every fixed \(\boldsymbol{\phi}\in\mathcal H\), the function \(\cL_n(\cdot ,\boldsymbol{\phi})\) is continuously differentiable. Moreover, there exists \(L>0\) such that
\[
\left\|
\nabla_{\bW}\cL_n(\bU;\boldsymbol{\phi})
-
\nabla_{\bW}\cL_n(\bV;\boldsymbol{\phi})
\right\|_F
\le
L\|\bU-\bV\|_F
\]
for all \(\bU,\bV\in\mathcal W_n\).
\end{assumption}

\begin{assumption}
\label{as:lower_bound}
There exists a constant \(\cL_\star>-\infty\), independent of the optimization iteration, such that
\[
\cL_n(\bW;\boldsymbol{\phi})
\ge
\cL_\star
\quad
\text{for every }
\bW\in\mathcal W_n
\text{ and }
\boldsymbol{\phi}\in\mathcal H.
\]
\end{assumption}

\begin{assumption}
\label{as:bounded_second_moment}
There exists a constant $\sigma^2<\infty$ such that, for every $t$,
\[
\mathbb E\left[
\frac{1}{nD}\|G_t\|_F^2
\middle|\mathcal F_t
\right]
\le
\sigma^2
\qquad\text{almost surely}.
\]
\end{assumption}

\begin{lemma}
\label{lem:adam_direction_bound}
Under Assumption~\ref{as:bounded_second_moment}, for
$\epsilon_{\bW}>0$ and $0\le\beta_{1,\bW},\beta_{2,\bW}<1$,
\[
\mathbb E\left[
\frac{1}{nD}\|\bd_t\|_F^2
\right]
\le
\frac{\sigma^2}{\epsilon_{\bW}^2},
\qquad t\ge1.
\]
\end{lemma}

\begin{proof}
The bias-corrected first-moment estimate can be written as
\[
\widehat{\bm}_{\bW,t}
=
\sum_{r=1}^t\omega_{t,r}G_r,
\qquad
\omega_{t,r}
=
\frac{(1-\beta_{1,\bW})\beta_{1,\bW}^{t-r}}
{1-\beta_{1,\bW}^t},
\]
where $\omega_{t,r}\ge0$ and
$\sum_{r=1}^t\omega_{t,r}=1$. By convexity,
\[
\left\|
\widehat{\bm}_{\bW,t}
\right\|_F^2
\le
\sum_{r=1}^t
\omega_{t,r}\|G_r\|_F^2.
\]
Taking expectations and using
Assumption~\ref{as:bounded_second_moment} gives
\[
\mathbb E\left[
\frac{1}{nD}
\left\|
\widehat{\bm}_{\bW,t}
\right\|_F^2
\right]
\le
\sigma^2.
\]
Since every entry of the Adam denominator is at least
$\epsilon_{\bW}$,
\[
\frac{1}{nD}\|\bd_t\|_F^2
\le
\frac{1}{nD\,\epsilon_{\bW}^2}
\left\|
\widehat{\bm}_{\bW,t}
\right\|_F^2.
\]
Therefore,
\[
\mathbb E\left[
\frac{1}{nD}\|\bd_t\|_F^2
\right]
\le
\frac{\sigma^2}{\epsilon_{\bW}^2}.
\]
\end{proof}

\section{Proof of Theorem~\ref{thm:biased_adam}}
\label{appendix: proof}

The Adam update satisfies
\[
\bW^{(t)}-\bW^{(t-1)}
=
-\alpha_{\bW,t}\bd_t.
\]
By Assumption~\ref{as:smoothness},
\begin{align*}
\cL_n(\bW^{(t)};\boldsymbol{\phi})
&\le
\cL_n(\bW^{(t-1)};\boldsymbol{\phi})-
\alpha_{\bW,t}
\left\langle
\nabla_{\bW}\cL_n
(\bW^{(t-1)};\boldsymbol{\phi}),\bd_t
\right\rangle_F
+
\frac{L\alpha_{\bW,t}^2}{2}\|\bd_t\|_F^2.
\end{align*}
Taking expectations and applying
Assumption~\ref{as:adam_alignment},
the law of total expectation, and
Lemma \ref{lem:adam_direction_bound} yield
\begin{align*}
\mathbb E[\cL_n(\bW^{(t)};\boldsymbol{\phi})]
&\le
\mathbb E[\cL_n(\bW^{(t-1)};\boldsymbol{\phi})]-
c_0\alpha_{\bW,t}
\mathbb E\!\left[
\left\|
\nabla_{\bW}\cL_n
(\bW^{(t-1)};\boldsymbol{\phi})
\right\|_F^2
\right]\\
&\quad+
c_b\alpha_{\bW,t}\mathbb E[\|b_t\|_F^2]
+
\frac{LnD\alpha_{\bW,t}^2\sigma^2}
{2\epsilon_{\bW}^2}.
\end{align*}

Summing over $t=1,\ldots,T$, substituting
$\alpha_{\bW,t}=\alpha/\sqrt T$, and using
Assumption~\ref{as:lower_bound}, we obtain
\begin{align*}
&\quad \frac{c_0\alpha}{\sqrt T}
\sum_{t=1}^{T}
\mathbb E\!\left[
\left\|
\nabla_{\bW}\cL_n
(\bW^{(t-1)};\boldsymbol{\phi})
\right\|_F^2
\right]\le
\mathbb E[\cL_n(\bW^{(0)};\boldsymbol{\phi})]
-\cL_\star
+
\frac{c_b\alpha}{\sqrt T}
\sum_{t=1}^{T}\mathbb E[\|b_t\|_F^2]
+
\frac{LnD\alpha^2\sigma^2}
{2\epsilon_{\bW}^2}.
\end{align*}
Dividing by $c_0\alpha nD\sqrt T$ and bounding the
minimum by the average gives
\[
\min_{0\le t\le T-1}
\mathbb E\!\left[\frac{1}{nD}
\left\|
\nabla_{\bW}\cL_n
(\bW^{(t)};\boldsymbol{\phi})
\right\|_F^2
\right]
\le
\frac{C_1}{\sqrt T}
+
\frac{c_b}{c_0T}
\sum_{t=1}^{T}\mathbb E\left[\frac{1}{nD}\|b_t\|_F^2\right].
\]
where
\[
C_1
=\frac{
\mathbb E[\cL_n(\bW^{(0)}; \boldsymbol \phi)]-\cL_\star
}{
c_0\alpha nD
}
+
\frac{L\alpha\sigma^2}{2c_0\epsilon_\bW^2}.
\]
For $m_0\le m_t\le n$ and $\bW^{(t-1)}\in\mathcal W_n$, Assumption~\ref{as:batch_bias} implies 
\[
\frac{1}{nD}\mathbb E[\|b_t\|_F^2]
\le C_b^2m_t^{-2\gamma}.
\]
Substitution into the preceding bound proves
\eqref{eq:convergence_bound}, with
$C_2=c_bC_b^2/c_0$.

\section{Implementation Details}\label{appendix:optimizer_settings}

This section provides the implementation and tuning settings used for
LaMP-DGP and the competing DGP methods in the numerical experiments.

\subsection{LaMP-DGP}
The optimization, batch-size, prediction, and nugget settings are described in Section~\ref{sec:numericalsetup}. For the mini-batch implementation, we randomly partition the observations into training and validation sets using an 80\%/20\% split. The latent-state batch size follows the schedule
\[
m_t = \min\left\{
n,\,
500,\,
\max\left(
150,
\left\lceil 15 t^{1/2}\right\rceil
\right)
\right\},
\]
with hyperparameter batch size \(m_\phi=500\), hyperparameter update frequency \(q_\phi=50\), validation frequency \(q_{\mathrm{val}}=10\), and prediction subset size \(m_{\mathrm{pred}}=500\). The prediction subset \(S_{\mathrm{pred}}\) is fixed throughout optimization and is used only for validation.  For scalable prediction, the Nystr\"om approximation  uses \(n_Q=1000\) training inputs, and the local second-layer GP uses \(100\) nearest neighbors. For the full batch implementation, we randomly partition the observations into training and validation sets using a 90\%/10\% split, with the hyperparameter update frequency \(q_\phi=25\) and the validation frequency \(q_{\mathrm{val}}=10\).

The learning rates and Adam hyperparameters used for the full-batch and mini-batch regimes are are summarized in Tables~\ref{tab:lamp_lr} and~\ref{tab:lamp_adam}. For the
mini-batch optimization, $\alpha_{\bW}$ is multiplied by $0.8$ every 250 epochs, while all other learning rates remain fixed. The parameters associated with $g$ apply only when the nugget is estimated. The Adam denominator constant $\epsilon_{\boldsymbol{\psi}}$ is shared across
parameter groups.

\begin{table}[htbp]
\centering
\caption{Learning rates used for LaMP-DGP optimization.}
\label{tab:lamp_lr}
\begin{tabular}{lcccc}
\toprule
Optimization
& $\alpha_{\bW}$
& $\alpha_{\boldsymbol\theta_1}$
& $\alpha_{\theta_2}$
& $\alpha_g$ \\
\midrule
Full-batch
& $0.1$ & $0.05$ & $0.05$ & $0.01$ \\
Mini-batch
& $0.05$ & $0.1$ & $0.1$ & $0.05$ \\
\bottomrule
\end{tabular}
\end{table}

\begin{table}[htbp]
\centering
\caption{Adam moment-decay and numerical-stability parameters for
LaMP-DGP. The two length-scale parameter groups share the same
moment-decay coefficients.}
\label{tab:lamp_adam}
\begin{tabular}{lccccccc}
\toprule
Optimization
& $\beta_{1,\bW}$
& $\beta_{2,\bW}$
& $\beta_{1,\boldsymbol\theta}$
& $\beta_{2,\boldsymbol\theta}$
& $\beta_{1,g}$
& $\beta_{2,g}$
& $\epsilon_{\boldsymbol{\psi}}$ \\
\midrule
Full-batch
& $0.9$ & $0.99$
& $0.5$ & $0.999$
& $0.9$ & $0.999$
& $10^{-4}$ \\
Mini-batch
& $0.9$ & $0.9$
& $0.5$ & $0.999$
& $0.9$ & $0.999$
& $10^{-4}$ \\
\bottomrule
\end{tabular}
\end{table}

\subsection{Competing DGP Methods}

\paragraph{MCMC with Vecchia (\texttt{MCMC-VEC}).}
We implement the MCMC-VEC method using \textsf{deepgp} \citep{booth2024deepgp}. We use a two-layer DGP
with a Vecchia approximation, latent dimension $D=d$, and
Mat\'ern kernels with smoothness $\nu=5/2$. We run 2000 MCMC iterations, discard the first
1000 as burn-in, and retain every second subsequent draw, yielding
500 posterior draws for prediction. This is shorter than the default MCMC
run in \textsf{deepgp} and is used to limit computational cost; longer
runs may improve posterior convergence at additional computational
expense.  For deterministic experiments,
the nugget is fixed at $10^{-8}$; for noisy experiments, it is
estimated. All implementations use 8 CPU cores for parallel computation.

\paragraph{DSVI using \textsf{GPyTorch}
(\texttt{DSVI-GPyTorch}).}
We implement doubly stochastic variational inference using
\textsf{GPyTorch} \citep{gardner2018gpytorch}. The model has two GP
layers with latent dimension $D=d$, squared-exponential kernels, linear hidden-layer mean functions, and a constant output-layer mean function. Each GP uses 100 inducing points, initialized from a standard normal
distribution and optimized during training. Models are trained with
Adam using learning rate $10^{-2}$ and batch size $\min(n,10000)$, for
2000 epochs when $n<500$ and 1000 epochs otherwise. We use one Monte
Carlo sample during training and 100 samples for prediction. The
hidden-layer white-noise variance is initialized at $2\times10^{-6}$
and estimated during training. The observation-noise variance is fixed
at $10^{-4}$ for deterministic experiments and estimated for noisy
experiments.

\paragraph{DSVI using \textsf{GPflux}
(\texttt{DSVI-GPflux}).}
We implement the method using \textsf{GPflux}
\citep{dutordoir2021gpflux}, with settings based on
\citet[Section~5]{sauer2023vecchia}. We use two GP layers with latent
dimension $D=d$ and Mat\'ern kernels with smoothness $\nu=5/2$. Each layer uses
$\min(n, 100)$ inducing locations, initialized by $k$-means and optimized
during training. Models are trained for 1000 epochs using Adam with a learning rate of $10^{-2}$ and the package’s default mini-batch size of 32. Following the reference implementation, the
observation-noise variance is initialized at $10^{-4}$ for deterministic, and estimated
during training for both deterministic and noisy experiments.

\paragraph{SGHMC (\texttt{SGHMC}).}
We use the authors' publicly available implementation
\citep{havasi2018sghmccode}. 
The model comprises two GP layers with hidden dimension $D=d$. Following the implementation's default specification, we use squared-exponential kernels, 100 inducing points per layer, and batch size $\min(n,10000)$. We modify the implementation to fix the observation-noise variance at $10^{-4}$ for deterministic experiments, while retaining noise estimation for noisy experiments.

\section{Synthetic Functions}\label{appendix:function}
We provide all the explicit functions used in the numerical studies described in Section~\ref{sec:numerical}.

\begin{itemize}

\item Schaffer Function \citep{simulationlib}: for
\(
\mathbf{x}\in[-2,2]^2
\), 
\[
f(\mathbf{x}) = 0.5 + \frac{\cos^2(\sin(|x_1^2 - x_2^2|))-0.5}{[1+0.001(x_1^2+x_2^2)]^2}.
\]

\item Park Function \citep{park1991tuning}: for $\mathbf{x} \in [0,1]^4$, 
\[
    f(\mathbf{x})=\frac{x_1}{2}\left(\sqrt{1+\frac{(x_2+x_3^2)x_4}{x_1^2}}-1\right)+\left(x_1+3x_4\right)\exp(1+\sin(x_3)).
\]

\item \(G\)-Function \citep{marrel2009calculations}: for
\(
\mathbf{x}\in[0,1]^d
\),
\begin{align*}
    &f(\mathbf{x}) = \prod_{i=1}^d \frac{|4x_i - 2|+a_i}{1+a_i}\\
    &\text{where}\quad a_i = \frac{i-2}{2},\, i = 1,\dots, d.
\end{align*}
In our experiments, we set $d = 4$ and $6$.

\item The Output Transformerless (OTL) Circuit Function \citep{ben2007modeling}: for
    \(
    \mathbf{x}
    =(R_{b1},\,R_{b2},\,R_f,\,R_{c1},\,R_{c2},\,\beta),
    \)
    with
    \[
    \begin{aligned}
    R_{b1} &\in [50,150], &
    R_{b2} &\in [25,70], &
    &R_f \in [0.5,3],\\
    R_{c1} &\in [1.2,2.5], &
    R_{c2} &\in [0.25,1.2], &
    &\beta \in [50,300].
    \end{aligned}
    \]
    The function is defined as
    \begin{align*}
    &f(\mathbf{x}) = \frac{(V_{b1}+0.74)\beta(R_{c2}+9)}{\beta(R_{c2}+9) + R_f} + \frac{11.35R_f}{\beta(R_{c2}+9) + R_f} + \frac{0.74R_f\beta(R_{c2}+9)}{(\beta(R_{c2}+9) + R_f)R_{c1}}\\ 
    & \text{where}\quad V_{b1} = \frac{12R_{b2}}{R_{b1}+R_{b2}}.
    \end{align*}

\item Piston Function \citep{ben2007modeling}: for
\(
\mathbf{x}
=
(M,S,V_0,k,P_0,T_a,T_0),
\)
with
\[
\begin{aligned}
M   &\in[30,60],             &
S   &\in[0.005,0.020],       &
V_0 &\in[0.002,0.010], &
k   &\in[1000,5000],         \\
P_0 &\in[90000,110000],      &
T_a &\in[290,296],           &
T_0 &\in[340,360].
\end{aligned}
\]
The function is defined as
\[
\begin{aligned}
    &f(\mathbf{x}) = 2\pi \sqrt{\frac{M}{k + S^2 \frac{P_0 V_0}{T_0}\frac{T_a}{V^2}}}\\
    &\text{where}\quad V = \frac{S}{2k}\left(\sqrt{A^2 + 4k \frac{P_0 V_0}{T_0}T_a} -A\right)\\
    &\qquad\quad \;\;A = P_0S+19.62M- \frac{kV_0}{S}
\end{aligned}
\]

 \item Borehole Function \citep{morris1993bayesian}: for
    \(
    \mathbf{x}
    =
    (r_w,\,r,\,T_u,\,H_u,\,T_l,\,H_l,\,L,\,K_w),
    \)
    with
    \[
    \begin{aligned}
    r_w & \in [0.05, 0.15], &
    r   & \in [100, 50000], &
    T_u & \in  [63070, 115600], &
    H_u & \in [990, 1110],  \\
    T_l & \in [63.1, 116],  &
    H_l & \in [700, 820],  &
    L & \in [1120, 1680], &
    K_w & \in [9855, 12045].
    \end{aligned}
    \]
    The function is defined as
    \[
f(\mathbf{x}) = \frac{2\pi T_u (H_u - H_l)}{\ln\left(\frac{r}{r_w}\right) \left[ 1 + \frac{2 L T_u}{\ln\left(\frac{r}{r_w}\right) r_w^2 K_w} + \frac{T_u}{T_l} \right]}.
\]

\item Wing Weight Function \citep{forrester2008engineering}:
for
\(
\mathbf{x}
=
(S_w,W_{fw},A,\Lambda,q,\lambda,t_c,N_z,W_{dg},W_p)
\), with
\[
\begin{aligned}
S_w     &\in[150,200],       &
W_{fw}  &\in[220,300],       &
A       &\in[6,10], &
\Lambda &\in[-10,10],        \\
q       &\in[16,45],         &
\lambda &\in[0.5,1], &
t_c     &\in[0.08,0.18],     &
N_z     &\in[2.5,6],         \\
W_{dg}  &\in[1700,2500],     &
W_p     &\in[0.025,0.08].
\end{aligned}
\]
where $\Lambda_r=\pi\Lambda/180$ is the sweep angle in radians. The function is defined as
\[
f(\mathbf{x)}=0.036S_w^{0.758}W_{fw}^{0.0035}\left(\frac{A}{\cos^2\Lambda_r}\right)^{0.6}q^{0.006}\lambda^{0.04}\left(\frac{100t_c}{\cos\Lambda_r}\right)^{-0.3}(N_zW_{dg})^{0.49}+S_wW_p.
\]

\item Turbulence-Inspired Function (Turb) \citep{rumsey2025all}:
for
\[
\mathbf{x}
=
\left(
l_r,\,
y_{\mathrm{wall}},\,
S,\,
\ell,\,
C_\mu,\,
A^+,\,
\beta,\,
\chi,\,
C_p,\,
C_d,\,
\alpha,\,
s
\right),
\]
with
\[
\begin{aligned}
l_r              &\in [4,8],          &
y_{\mathrm{wall}} &\in [0.001,0.2],   &
S                 &\in [0.05,10],\\
\ell              &\in [0.01,0.5],    &
C_\mu             &\in [0.05,0.3],    &
A^+               &\in [2,25],\\
\beta             &\in [-2,2],        &
\chi              &\in [-3,3],        &
C_p               &\in [0.2,2],\\
C_d               &\in [0.2,2],       &
\alpha            &\in [0,2],         &
s                 &\in [0.1,5].
\end{aligned}
\]
The function is defined as
\[
\begin{aligned}
&f(\mathbf{x})
=
\frac{P}{(1+\beta^2)D}
+
\frac{\nu_t}{2(1+\nu_t)}
+
0.1\,\tanh(\chi)\log(1+l_r)
+
0.2\,\frac{s^2}{1+s^2+\alpha^2}
+
0.05\,\alpha f_w,\\
&\text{where}\quad f_w
=
\big(
1-\exp\!\left(-A^+y_{\mathrm{wall}}\right)
\big)^2,\;
P=
C_p |S|^{3/2}\ell
\left(1+\frac{1}{4}\tanh\beta\right),\\
&\qquad \quad \;\;\; D
=
C_d\left(1+\ell^2\right)\left(1+\alpha^2\right),\;
\nu_t
=
\frac{C_\mu\ell^2|S|f_w}{1+|\beta|}.
\end{aligned}
\]
\item Oakley--O'Hagan Function \citep{oakley2004probabilistic}: for
$\bx\in[-3,3]^{15}$,
the function is then defined as
\[
\begin{aligned}
&f(\mathbf{x})=\mathbf{a}_1^\top\mathbf{x}+\mathbf{a}_2^\top\sin(\mathbf{x})+\mathbf{a}_3^\top\cos(\mathbf{x})+\mathbf{x}^\top M\mathbf{x},\\
& \mathbf{a}_1=(0.0118,0.0456,0.2297,0.0393,0.1177,0.3865,0.3897,0.6061,\\
&\qquad \qquad 0.6159,0.4005,1.0741,1.1474,0.7880,1.1242,1.1982)^\top\\
&\mathbf{a}_2=(0.4341,0.0887,0.0512,0.3233,0.1489,1.0360,0.9892,0.9672,\\
&\qquad \qquad 0.8977,0.8083,1.8426,2.4712,2.3946,2.0045,2.2621)^\top\\
&\mathbf a_3
=(0.1044,0.2057,0.0774,0.2730,0.1253,
        0.7526,0.8570,1.0331,\\
       &\qquad \qquad 0.8388,0.7970,2.2145,2.0382,2.4004,
        2.0541,1.9845)^\top.
\end{aligned}
\]
The matrix $M$ has rows $\mathbf m_i^\top$, $i=1,\ldots,15$,
where
\begin{align*}
\mathbf{m}_{1}={}&(-0.0225,-0.1850,0.1342,0.3687,0.1717,0.1365,-0.4403,-0.0814,\\*
&\qquad 0.7132,-0.4436,0.5038,-0.0241,-0.0459,0.2167,0.0559)^\top,\\[2pt]
\mathbf{m}_{2}={}&(0.2566,0.0538,0.2580,0.2380,-0.5913,-0.0816,-0.2875,0.4158,\\*
&\qquad 0.4975,0.0839,-0.1106,0.0332,-0.1398,-0.0310,-0.2232)^\top,\\[2pt]
\mathbf{m}_{3}={}&(-0.0560,0.1954,0.0955,-0.2863,-0.1444,0.2237,0.1453,0.2900,\\*
&\qquad 0.2311,-0.3193,-0.2904,-0.2096,0.4314,0.0244,0.0449)^\top,\\[2pt]
\mathbf{m}_{4}={}&(0.6645,0.4307,0.2992,-0.1620,-0.3148,-0.3903,0.1768,0.0580,\\*
&\qquad 0.1723,0.1347,-0.3528,0.2515,-0.0188,0.3648,-0.3250)^\top,\\[2pt]
\mathbf{m}_{5}={}&(-0.1213,0.1246,0.1066,0.0466,-0.2168,0.1949,-0.0655,0.0244,\\*
&\qquad -0.0968,0.1937,0.3335,0.3130,-0.0836,-0.2534,0.3733)^\top,\\[2pt]
\mathbf{m}_{6}={}&(-0.2838,-0.3282,-0.1050,-0.2207,-0.1371,-0.1443,-0.1150,0.2242,\\*
&\qquad -0.0304,-0.5151,0.0173,0.0390,0.3607,0.3090,0.0500)^\top,\\[2pt]
\mathbf{m}_{7}={}&(-0.0779,0.0037,0.8869,-0.2659,-0.0793,-0.0427,-0.1865,-0.3560,\\*
&\qquad -0.1750,0.0887,0.4003,-0.0560,0.1372,0.2149,-0.0113)^\top,\\[2pt]
\mathbf{m}_{8}={}&(-0.0923,0.5921,0.0313,-0.0331,-0.2431,-0.0998,0.0345,0.0951,\\*
&\qquad -0.3380,0.0064,-0.6121,0.0813,0.8868,0.1425,0.1478)^\top,\\[2pt]
\mathbf{m}_{9}={}&(-0.1319,0.5288,0.1265,0.0451,0.5837,0.3729,0.1140,-0.2948,\\*
&\qquad -0.5701,0.4629,-0.0941,0.1396,-0.3861,-0.4490,-0.1460)^\top,\\[2pt]
\mathbf{m}_{10}={}&(0.0581,-0.3229,0.0931,0.0724,-0.5692,0.5255,0.2366,-0.0118,\\*
&\qquad 0.0718,0.0783,-0.1336,0.2272,0.1437,-0.4520,-0.5557)^\top,\\[2pt]
\mathbf{m}_{11}={}&(0.6615,0.3463,0.1410,0.5188,-0.2802,-0.1603,-0.0684,-0.2043,\\*
&\qquad 0.0697,0.2311,-0.0444,-0.1646,0.2162,0.0043,-0.0874)^\top,\\[2pt]
\mathbf{m}_{12}={}&(0.3160,-0.0276,0.1343,0.1350,0.0540,-0.1737,0.1753,0.0603,\\*
&\qquad -0.1791,-0.3106,-0.2536,0.0258,-0.4301,-0.6227,-0.0340)^\top,\\[2pt]
\mathbf{m}_{13}={}&(-0.2904,0.0341,0.0349,-0.1212,0.0260,-0.3355,-0.4142,0.0532,\\*
&\qquad -0.2710,-0.0263,0.4102,0.2664,0.1558,-0.1867,0.0199)^\top,\\[2pt]
\mathbf{m}_{14}={}&(-0.2439,-0.4410,0.0126,0.2495,0.0711,0.2462,0.1748,0.0085,\\*
&\qquad 0.2515,-0.1466,-0.0846,0.3693,-0.2996,0.1104,-0.7569)^\top,\\[2pt]
\mathbf{m}_{15}={}&(0.0415,-0.2598,0.4640,-0.3611,-0.9498,-0.1650,0.0031,0.0528,\\*
&\qquad 0.2252,0.3839,0.4556,-0.1863,0.0082,0.1667,0.1605)^\top.
\end{align*}

\end{itemize}

\section{Simulation with Noise}\label{appendix:noise}

To assess performance beyond the deterministic setting considered in
the main numerical study, we conduct additional experiments on two
noisy benchmark problems: the four-dimensional Park function and the
four-dimensional $G$-function. In both cases, we standardize the noise-free responses to unit variance and perturb the training responses by independent Gaussian noise, $\epsilon_i \sim \mathcal{N}(0, 0.01^2)$. Thus, the noise-to-signal ratio (NSR), defined as the ratio of the noise
standard deviation to the signal standard deviation, is $0.01$. Figure~\ref{fig:park_g4_noise} summarizes the performance of all five emulators.



\begin{figure}[htbp]
    \centering
    \includegraphics[width=\linewidth]{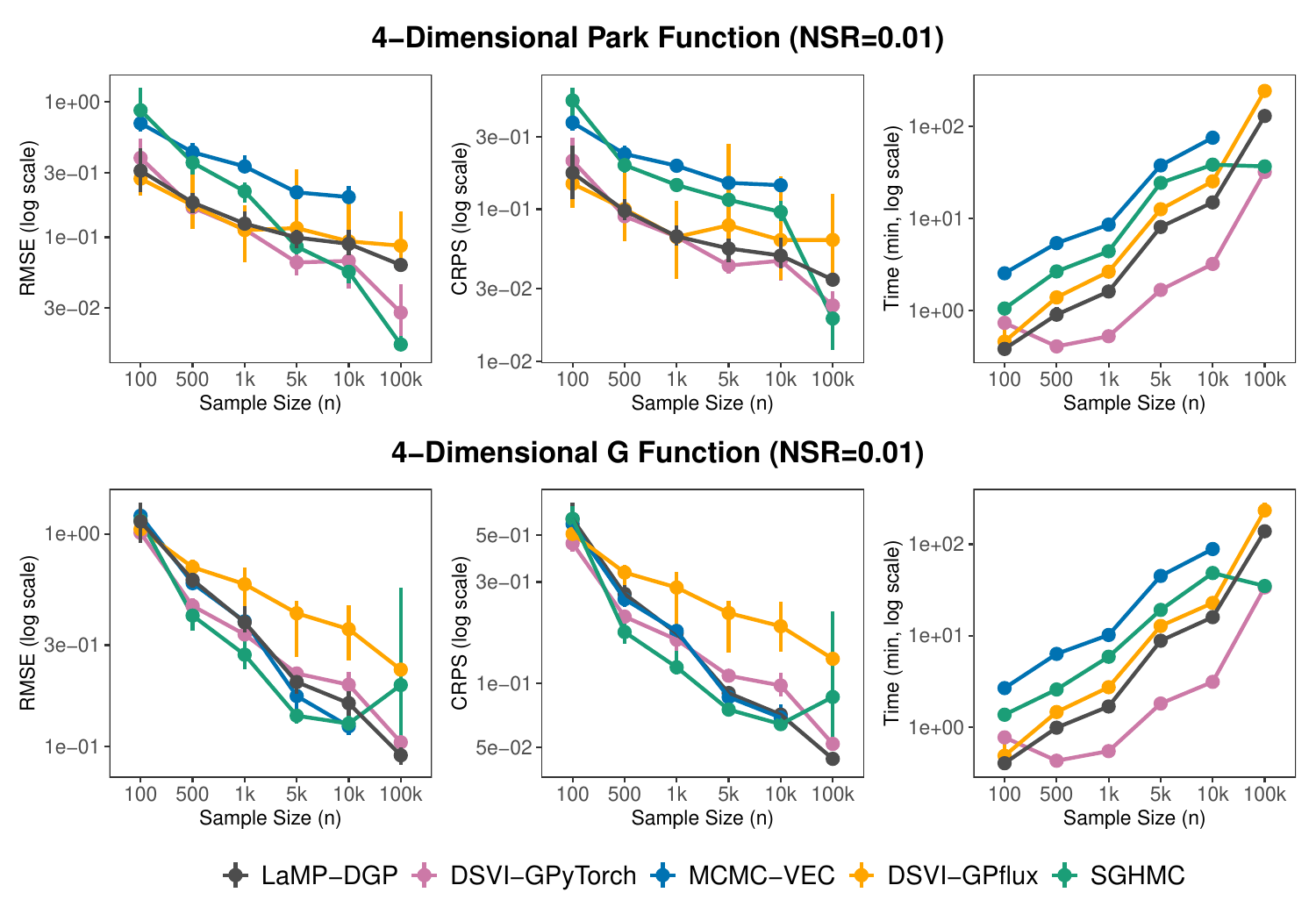}\centering
    \caption{Predictive performance of the five emulators on the Park function (top) and the \(G\)-function (bottom) under a noise-to-signal ratio (NSR) of 0.01. Points represent means across 10 independent repetitions, and vertical lines indicate the corresponding minimum-to-maximum ranges.}
    \label{fig:park_g4_noise}
\end{figure}

\end{document}